\pdfoutput=1

\documentclass[11pt]{article}

\usepackage[]{ACL2023}

\usepackage{times}
\usepackage{latexsym}
\usepackage{url}
\usepackage{booktabs}
\usepackage{adjustbox}
\usepackage{colortbl}
\usepackage{threeparttable}
\usepackage{makecell}
\usepackage{graphicx}
\usepackage{multirow}
\usepackage{wrapfig}
\usepackage{amsmath}
\usepackage{amssymb}
\usepackage{amsthm}
\usepackage{capt-of}
\usepackage{soul}

\newtheorem{proposition}{Proposition}

\newtheorem{theorem}{Theorem}

\newtheorem{remark}{Remark}

\usepackage{pifont}
\newcommand{\cmrk}{\textcolor[rgb]{0,0.7,0}{\ding{51}}}
\newcommand{\xmrk}{\textcolor{red}{\ding{55}}}

\usepackage[T1]{fontenc}

\usepackage[utf8]{inputenc}

\usepackage{microtype}

\usepackage{inconsolata}

\title{Lifting the Curse of Capacity Gap in \\Distilling Language Models}

\author{Chen Zhang\textsuperscript{\ding{168}}, Yang Yang\textsuperscript{\ding{169}}, Jiahao Liu\textsuperscript{\ding{169}}, Jingang Wang\textsuperscript{\ding{169}}, Yunsen Xian\textsuperscript{\ding{169}}, \\ 
  \textbf{Benyou Wang\textsuperscript{\ding{170}}, Dawei Song\textsuperscript{\ding{168}}\Thanks{ Dawei Song is the corresponding author.}} \\
  \textsuperscript{\ding{168}}Beijing Institute of Technology \quad
  \textsuperscript{\ding{169}}Meituan NLP \\
  \textsuperscript{\ding{170}}The Chinese University of Hong Kong, Shenzhen \\
  \texttt{chenzhang9702@outlook.com}}

\begin{document}

\maketitle

\begin{abstract}
Pretrained language models (LMs) have shown compelling performance on various downstream tasks, but unfortunately they require a tremendous amount of inference compute. Knowledge distillation finds a path to compress LMs to small ones with a teacher-student paradigm. However, when the capacity gap between the teacher and the student is large, a curse of capacity gap appears, invoking a deficiency in distilling LMs. While a few studies have been carried out to fill the gap, the curse is not yet well tackled. In this paper, we aim at lifting the curse of capacity gap via enlarging the capacity of the student without notably increasing the inference compute. Largely motivated by sparse activation regime of mixture of experts (\textsc{MoE}), we propose a mixture of minimal experts (\textsc{MiniMoE}), which imposes extra parameters to the student but introduces almost no additional inference compute. Experimental results on GLUE and CoNLL demonstrate the curse of capacity gap is lifted by the magic of \textsc{MiniMoE} to a large extent. \textsc{MiniMoE} also achieves the state-of-the-art performance at small FLOPs compared with a range of competitive baselines. With a compression rate as much as $\sim$50$\times$, \textsc{MiniMoE} preserves $\sim$95\% GLUE score of the teacher.\footnote{Code is available at \url{https://github.com/GeneZC/MiniMoE}}
\end{abstract}

\section{Introduction}

Pretrained language models (LMs) have become a popular choice for various downstream tasks, e.g., text classification, token classification, and question answering~\citep{DevlinCLT19,Liu19,RaffelSRLNMZLL20}. Unfortunately, appealing performance comes with a huge cost of inference compute due to the scale of LMs. Knowledge distillation~\citep{HintonVD15,SunCGL19}, as an alternative to model pruning~\citep{HanPTD15} and quantization~\citep{SungSH15}, discovers a way to compress~\citep{BucilaCN06} LMs with a teacher-student paradigm.

However, in LM distillation, we recognize a \textit{\textbf{curse} of capacity gap} as:
\begin{quote}
    \centering
    ``\textit{Large teachers, poor students.}''
\end{quote}
The curse of capacity gap refers to a deficiency that a larger teacher might unexpectedly result in a poorer student especially when the capacity gap between the teacher and the student is large~\citep{MirzadehFLLMG20,ChoH19}, as illustrated in Table~\ref{tab_1}. Notably, this is the first verification in LM distillation since previous studies recognize the curse in vision model distillation. Although a few studies~\citep{WangW0B0020,Zhang22,ParkCJKH21} have investigated to fill the gap, the curse is still not yet tackled.

\begin{table}[t]
    \centering
    \captionof{table}{The \textit{curse of the capacity gap} in terms of GLUE~\citep{WangSMHLB19}. The $\triangle$ denotes the performance difference of preceding two numbers. To ensure students at similar scales, the student/teacher scale ratios are properly reduced for some methods.}
    \begin{adjustbox}{width=0.47\textwidth,center}
    \begin{tabular}{lccc}
    \toprule
      \textbf{Method} & BERT\textsubscript{\sf base} & BERT\textsubscript{\sf large} & $\triangle$ \\
      \midrule
      Teacher & 86.7 & 88.3 & $+$1.6 \\
      \midrule
      KD\textsubscript{\sf 10\%/5\%}~\citeyearpar{HintonVD15} & 81.3 & 80.8 & $-$0.5 \\
      DynaBERT\textsubscript{\sf 15\%/5\%}~\citeyearpar{HouHSJCL20} & 81.1 & 79.2 & $-$1.9 \\
      MiniDisc\textsubscript{\sf 10\%/5\%}~\citeyearpar{Zhang22} & 82.4 & 82.1 & $-$0.3 \\
      TinyBERT\textsubscript{\sf 4L;312H}~\citeyearpar{JiaoYSJCL0L20} & 82.7 & 82.5 & $-$0.2 \\
      MiniLM\textsubscript{\sf 3L;384H}~\citeyearpar{WangBHDW21} & 82.5 & 82.0 & $-$0.5 \\
      MiniMoE\textsubscript{\sf 3L;384H} (\textbf{ours}) & 82.6 & 83.1 & $+$0.5 \\
    \bottomrule
    \end{tabular}
    \end{adjustbox}
    \label{tab_1}
\end{table}
        
\begin{figure}
    \centering
    \includegraphics[width=0.47\textwidth]{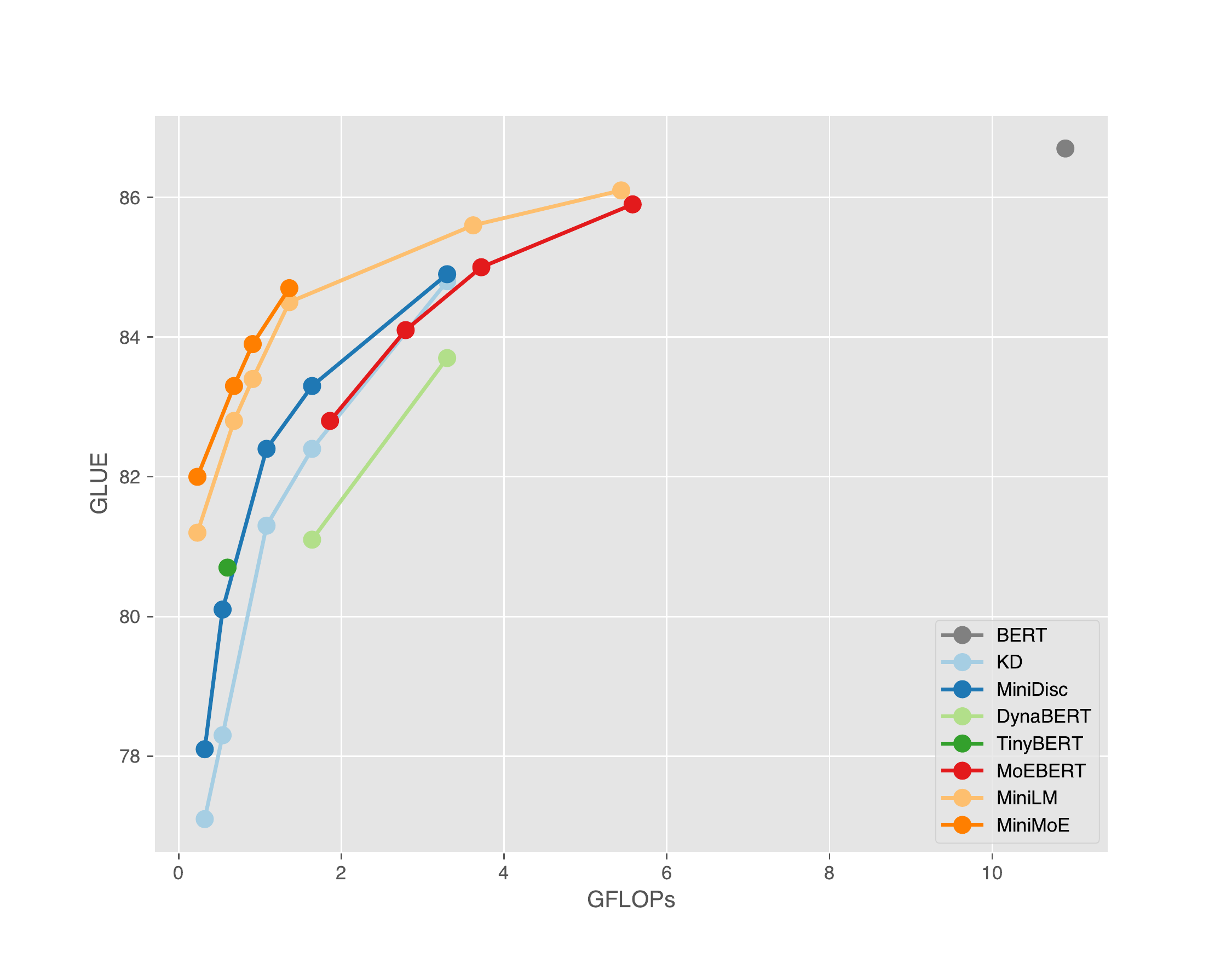}
    \caption{GLUE v.s. GFLOPs.}
    \label{fig_1}
\end{figure}

To the demand, we aim at lifting the curse of capacity gap via enlarging the capacity of the student without notably increasing the inference compute. We propose a mixture of minimal experts (\textsc{MiniMoE}), inspired by the intuition of sparse activation of mixture of experts (\textsc{MoE})~\citep{ShazeerMMDLHD17}. Thanks to that the activation process can be parallel on either single or multiple devices~\citep{He21,RajbhandariLYZA22}, \textsc{MiniMoE} on the one hand imposes extra parameters to the student, but on the other hand introduces negligibly additional inference compute brought by routing algorithm. To our best knowledge, this is the first work aiming at lifting the curse completely.

Experiments are conducted on GLUE~\citep{WangSMHLB19} and CoNLL~\citep{SangM03}. The results exhibit that \textsc{MiniMoE} largely lifts the curse of the gap as in Table~\ref{tab_1}. \textsc{MiniMoE} also achieves state-of-the-art performance compared with a range of competitive baselines, as shown in Figure~\ref{fig_1}. With compression as much as $\sim$50$\times$, \textsc{MiniMoE} preserves $\times$95\% GLUE score of the teacher. Thereby, we state that \textsc{MiniMoE} is \textit{a small yet nontrivial magic, making a great difference in lifting the curse}.


\section{Curse of Capacity Gap}

The curse of capacity gap is not new but is already recognized in studies on vision model distillation~\citep{MirzadehFLLMG20,ChoH19}. While a hit-the-mind drawback of the curse is that the performance of distilling to a small student can be dramatically worse than that of distilling to a slightly larger one, a rather counter-intuitive deficiency is invoked as that the performance of distilling from a large teacher can be unexpectedly worse than that of distilling from a smaller one (i.e., \textit{large teacher, poor student}). We here give a minor theoretical justification on the curse, as a plus to the empirical justification.

\begin{proposition}[VC dimension theory,~\citealp{Vapnik98}]
\label{propos_1}
Assuming that the teacher function is $f_{\mathcal{T}}\in\mathcal{F}_{\mathcal{T}}$, the labeling function is $f\in\mathcal{F}$, and the data is $\mathcal{D}$, we have:
\begin{equation}\nonumber
    r(f_{\mathcal{T}})-r(f)\leq \epsilon_{\mathcal{T}}+o(\frac{|\mathcal{F}_{\mathcal{T}}|_{c}}{|\mathcal{D}|}),
\end{equation}
where $r(\cdot)$ is the risk function, $|\cdot|_{c}$ is the function class capacity measure, and $|\cdot|$ is the data scale measure. It should be highlighted that the approximation error $\epsilon_{\mathcal{T}}$ is negatively correlated with the capacity of the teacher model while the estimation error $o(\cdot)$ is correlated with the learning optimization.
\end{proposition}

\begin{proposition}[Generalized distillation theory,~\citealp{Lopez-PazBSV15}]
\label{propos_2}
Additionally providing that the student function is $f_{\mathcal{S}}\in\mathcal{F}_{\mathcal{S}}$, we have:
\begin{equation}\nonumber
    r(f_{\mathcal{S}})-r(f_{\mathcal{T}})\leq \epsilon_{\mathcal{G}}+o(\frac{|\mathcal{F}_{\mathcal{S}}|_{c}}{|\mathcal{D}|^{\alpha}}),
\end{equation}
where the approximation error $\epsilon_{\mathcal{G}}$ is positively correlated with the capacity gap between the teacher and the student models, and $1/2\leq\alpha\leq 1$ is a factor correlated to the learning rate.
\end{proposition}

\begin{theorem}
The bound for the student function at a learning rate can be written as:
\begin{equation}\nonumber
\begin{aligned}
    r(f_{\mathcal{S}})-r(f)&\leq \epsilon_{\mathcal{T}}+\epsilon_{\mathcal{G}}+o(\frac{|\mathcal{F}_{\mathcal{T}}|_{c}}{|\mathcal{D}|})+o(\frac{|\mathcal{F}_{\mathcal{S}}|_{c}}{|\mathcal{D}|^{\alpha}}) \\
    &\leq \epsilon_{\mathcal{T}}+\epsilon_{\mathcal{G}}+o(\frac{|\mathcal{F}_{\mathcal{T}}|_{c}+|\mathcal{F}_{\mathcal{S}}|_{c}}{|\mathcal{D}|^{\alpha}}),
\end{aligned}
\end{equation}
\end{theorem}
\begin{proof}
The proof is rather straightforward by combining Proposition~\ref{propos_1} and~\ref{propos_2}.
\end{proof}

\begin{remark}
\label{rmk_1}
Under the same distillation setting, we can ignore the estimation error. When we compare two students of different capacities distilled from a teacher of the same capacity, the student of a smaller capacity has a larger $\epsilon_{\mathcal{G}}$ thus lower performance. When we compare two students of the same capacities distilled from teachers of different capacities, the student distilled from the teacher of a larger capacity has a smaller $\epsilon_{\mathcal{T}}$ yet a larger $\epsilon_{\mathcal{G}}$ thus a tradeoff. 
\end{remark}

Remark~\ref{rmk_1} basically tells that a tradeoff is associated with the increase of teacher capacity, implying that increasing teacher capacity would first lead to improved but then degraded student performance. This tradeoff naturally corresponds with the curse.

On the other hand, it is accepted that large capacity gap is a pain and is processed in literature of LM distillation~\citep{WangW0B0020,Zhang22,ZhouXM22}. Being unaware of the curse of capacity gap, these studies attempt to offer student-friendly teachers by either interpolating teacher assistants~\citep{WangW0B0020,Zhang22} or adapting teacher knowledge~\citep{ZhouXM22,Yang22}. The unawareness is largely due to a fun fact that they only distil LMs like BERT\textsubscript{\sf base}, but neglect the scalability to LMs like BERT\textsubscript{\sf large} especially when the student is small. Though the performance of student can be boosted in this way, the curse still remains in LM distillation as in Figure~\ref{fig_2}. Other related work in knowledge distillation is given in Appendix~\ref{app_e}.

\begin{figure}[ht]
    \centering
    \includegraphics[width=0.47\textwidth]{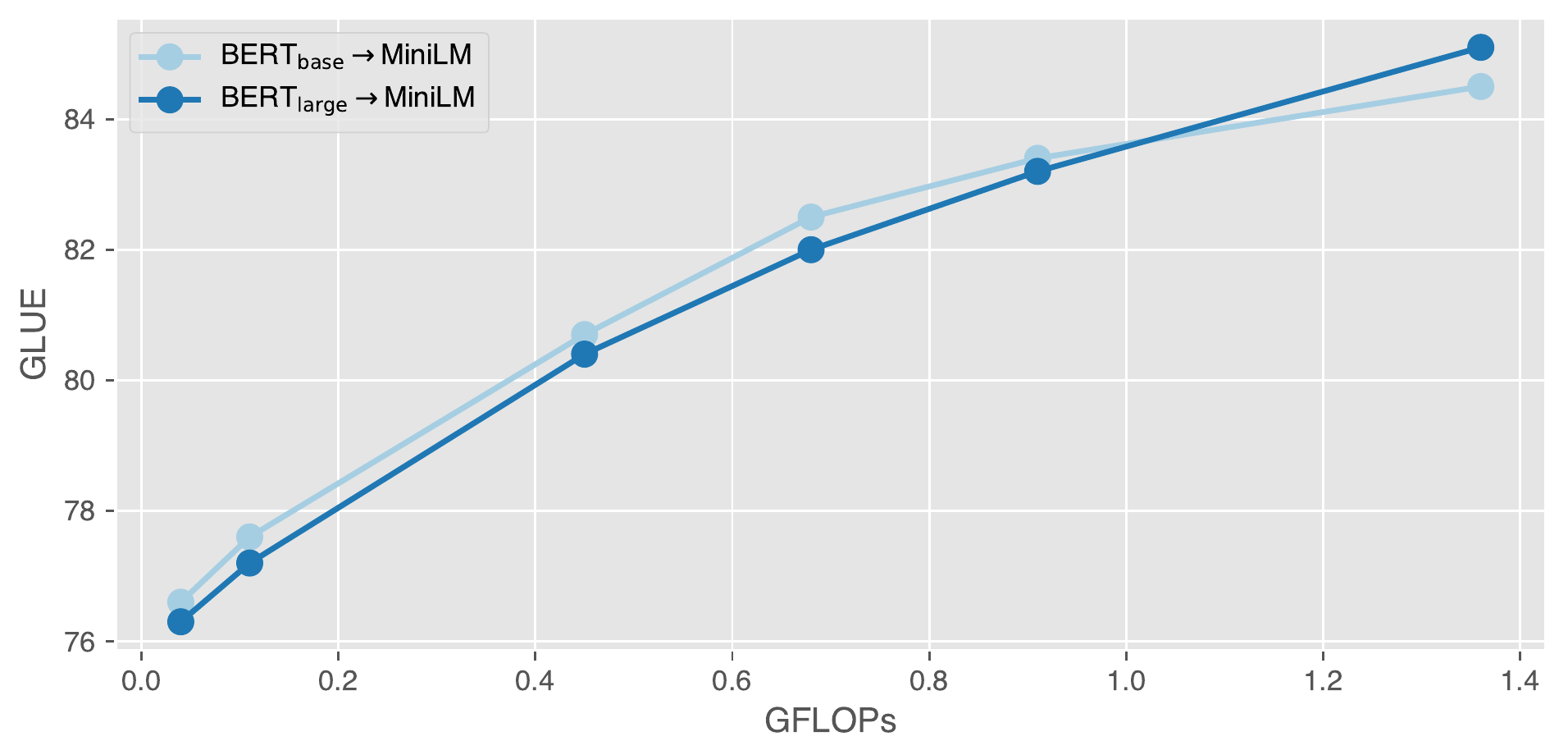}
    \includegraphics[width=0.47\textwidth]{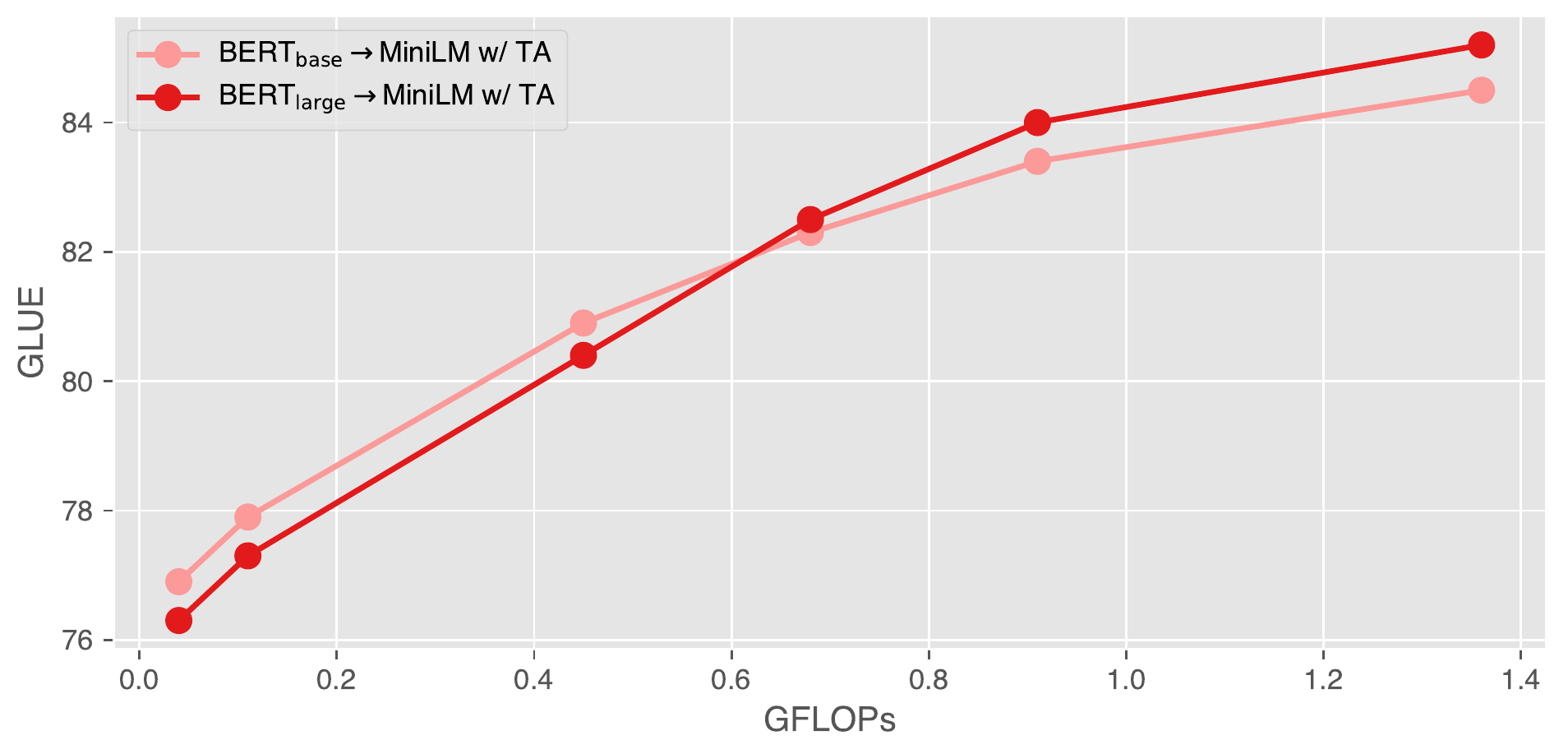}
    \caption{The performance of MiniLM and MiniLM w/ TA across different student scales upon distilling BERT\textsubscript{\sf base}. We are glad to share checkpoints of an array of scales, together with those of \textsc{MiniMoE}, to facilitate the development of related research. It should be noted the unit of a vertical grid is comparably large.}
    \label{fig_2}
\end{figure}

Embarrassingly, while the curse is claimed to be tackled in vision model distillation~\citep{ZhuW21a,ParkCJKH21,Zhao22}, our preliminary study (cf. Table~\ref{tab_6} in Appendix~\ref{app_h}) indicates they are either expensive or not capable of LMs. The potential differences are as follows: tasks (e.g., ImageNet v.s. GLUE), backbones (e.g., ResNets v.s. transformers), and paradigms (e.g., from scratch v.s. pretraining). 


\section{MiniMoE}

\subsection{Motivation}

Enlarging the capacity of the student is an intuitive solution to lift the curse of capacity gap. However, regarding the inference compute efficiency, the increase of capacity should not introduce much inference compute. 

An initial proposal can be using quantized backbones~\citep{ZafrirBIW19,BaiZHSJJLLK20}. Quantized backbones may decrease the compute precision, therefore maintaining inference compute constant, along the course of enlarging the capacity. But a vital portion of hardware-specific modifications are needed to do so. We hence move on to next possibility. 

Another alternative is using dynamic networks~\citep{Han21} based on the idea of conditional computation~\citep{BengioBPP15}. MoE computation~\citep{ShazeerMMDLHD17,Fedus21} is an option derived upon the sparse activation property to increase the scale with only minor losses in compute efficiency. The other commonly used one is depth-adaptive computation~\citep{XinTLYL20,ZhouXGM0W20,GoyalCRCSV20,KimC20} which involves layers into computation adaptively on either example~\citep[alias early exiting,][]{XinTLYL20,ZhouXGM0W20} or token~\citep[alias token reduction,][]{GoyalCRCSV20,KimC20} level. A critical distinction between MoE and depth-adaptive models is that the compute of an MoE model is accurately under control while that of a depth-adaptive model is not. We are impelled by the merits of MoE, and propose a \textsc{MiniMoE} so that the capacity of the student can be enlarged without much inference overhead increment.

Additionally, we argue that \textsc{MiniMoE} is orthogonal to alternatives mentioned above, and \textsc{MiniMoE} can be incorporated to these alternatives and makes it possible to serve more extreme scenarios. It is noteworthy that a certain stream of work~\citep{ZhangL00S022,ZuoZLHZC22} actually accelerates LMs via precisely converting them into MoE models. Nonetheless, the moefication process is directly exerted to LMs with limited inference compute improvements (cf. MoEBERT in Figure~\ref{tab_1}). Contrarily, \textsc{MiniMoE} is comprised of minimal experts, each of which can be extremely small. A comparison between mentioned possibilities and \textsc{MiniMoE} is listed in Table~\ref{tab_2}. And other related work of interest is given in Appendix~\ref{app_e}.

\begin{table}[t]
    \centering
    \captionof{table}{A comparison between \textsc{MiniMoE} and other possible alternatives.}
    \begin{adjustbox}{width=0.47\textwidth,center}
    \begin{tabular}{cccc}
    \toprule
      \textbf{Method} & \makecell[c]{\textbf{Flexible}\\ \textbf{Hardware}} & \makecell[c]{\textbf{Controllable}\\ \textbf{Compute}} & \makecell[c]{\textbf{Scalable}\\ \textbf{Compute}} \\
    \midrule
      Quantization & \xmrk & \cmrk & \cmrk \\
      Depth-adaptation & \cmrk & \xmrk & \cmrk \\
      MoEfication & \cmrk & \cmrk & \xmrk \\
      \textsc{MiniMoE} & \cmrk & \cmrk & \cmrk \\
    \bottomrule
    \end{tabular}
    \end{adjustbox}
    \label{tab_2}
\end{table}

\subsection{Implementation}

\paragraph{Minimal Language Models}

Typical language models are comprised of a stack of transformers layers~\citep{VaswaniSPUJGKP17}, and are pretrained with language modeling tasks such as masked language modeling~\citep{DevlinCLT19}. A transformer layer can be decomposed to a multi-head self-attention (MHA) block and a feed-forward network (FFN) block. Concretely, given an $n$-length sequence of $d$-dimension input vectors $\mathbf{X}\in\mathbb{R}^{n\times d}$ with the $i$-th vector being $\mathbf{x}_{i}$, the output of the MHA block with $A$ independent heads can be represented as:
\begin{equation}\nonumber
    \text{MHA}(\mathbf{X})=\sum_{j=1}^{A}\text{Attn}(\mathbf{X};\mathbf{W}^{\sf Q}_{j},\mathbf{W}^{\sf K}_{j})\mathbf{X}\mathbf{W}^{\sf V}_{j}\mathbf{W}^{\sf O}_{j},
\end{equation}
\begin{equation}\nonumber
\begin{aligned}
    \text{Attn}(\mathbf{X};&\mathbf{W}^{\sf Q}_{j},\mathbf{W}^{\sf K}_{j})= \\
    &\text{softmax}(\mathbf{X}\mathbf{W}^{\sf Q}_{j}\mathbf{W}^{\sf K\top}_{j}\mathbf{X}^{\top}/d^{\sf A}),
\end{aligned}
\end{equation}
where the $j$-th head is parameterized by $\mathbf{W}^{\sf Q}_{j}$, $\mathbf{W}^{\sf K}_{j}$, $\mathbf{W}^{\sf V}_{j}\in\mathbb{R}^{d\times d^{\sf A}}$, and $\mathbf{W}^{\sf O}_{j}\in\mathbb{R}^{d^{\sf A}\times d}$. On the other hand, the output of the FFN block is shown as:
\begin{equation}\nonumber
    \text{FFN}(\mathbf{X})=\text{GELU}(\mathbf{X}\mathbf{W}^{\sf I})\mathbf{W}^{\sf O},
\end{equation}
where two fully-connected layers are parameterized by $\mathbf{W}^{\sf I}\in\mathbb{R}^{d\times d^{\sf I}}$ and $\mathbf{W}^{\sf O}\in\mathbb{R}^{d^{\sf I}\times d}$ respectively. Details like biases, normalizations of a transformer layer are omitted for brevity.

To reach an acceptable compute budget, pioneering studies either pretrain language models or distil ones of small scales from LMs as in Figure~\ref{fig_x}. There are three lines of work in LM distillation: firstly, task-specific distillation~\citep{SunCGL19,LiLZXYJ20,SunGFCWL20,ParkKY21,HouHSJCL20,XiaZC22} that conducts distillation on a specific task at finetuning stage; secondly, task-agnostic distillation~\citep{Turc19,Sanh19,SunYSLYZ20,WangBHDW21} that conducts distillation at pretraining stage; and thirdly, two-stage distillation~\citep{JiaoYSJCL0L20} that combines the power of both task-agnostic and -specific distillation. Here, the distilled language models only refer to language models distilled with task-agnostic distillation regarding better task-scalability as the number of concerned tasks explodes. 

We formally define the distilled language models as minimal language models~\citep[MiniLMs, somehow abuse of notation with][]{WangW0B0020} notated with $\mathcal{S}$. In contrast, LMs are notated with $\mathcal{T}$. The learning objective of MiniLMs can be abstracted as $\mathcal{L}(\mathcal{S};\mathcal{T},\mathcal{D})$, where $\mathcal{D}$ denotes the data. The specific form of $\mathcal{L}$ can be adapted to arbitrary alignment strategies. We adopt a relation alignment strategy~\citep{WangBHDW21} as follows:
\begin{equation}\nonumber
\begin{aligned}
    &\mathcal{L}(\mathcal{S};\mathcal{T},\mathcal{D})=\mathbb{E}_{\mathbf{X}\sim\mathcal{D}}\sum_{j=1}^{R} \\
    &\text{KL}(\text{Reln}(\mathbf{X};{}^{\mathcal{T}}\mathbf{W}^{\sf Q}_{j}),\text{Reln}(\mathbf{X};{}^{\mathcal{S}}\mathbf{W}^{\sf Q}_{j})) \\
    &+\text{KL}(\text{Reln}(\mathbf{X};{}^{\mathcal{T}}\mathbf{W}^{\sf K}_{j}),\text{Reln}(\mathbf{X};{}^{\mathcal{S}}\mathbf{W}^{\sf K}_{j})) \\
    &+\text{KL}(\text{Reln}(\mathbf{X};{}^{\mathcal{T}}\mathbf{W}^{\sf V}_{j}),\text{Reln}(\mathbf{X};{}^{\mathcal{S}}\mathbf{W}^{\sf V}_{j})),
\end{aligned}
\end{equation}
\begin{equation}\nonumber
\begin{aligned}
    \text{Reln}(\mathbf{X};&{}^{\mathcal{T}}\mathbf{W}^{\sf Q}_{j}) \\
    &=\text{softmax}(\mathbf{X}{}^\mathcal{T}\mathbf{W}^{\sf Q}_{j}{}^\mathcal{T}\mathbf{W}^{\sf Q\top}_{j}\mathbf{X}^{\top}/d^{\sf R}),
\end{aligned}
\end{equation}
where KL stands for kullback-leibler divergence. Essentially, relation heads are derived by merging the original $A$ attention heads and then splitting them to $R$ heads. ${}^{\mathcal{T}}\mathbf{W}^{\sf Q}_{j}$ is the redistributed query parameter of the $j$-th relation head within totally $R$ heads from the last layer of the LM, likewise ${}^{\mathcal{T}}\mathbf{W}^{\sf K}_{j}$ and ${}^{\mathcal{T}}\mathbf{W}^{\sf V}_{j}$ are the key and value parameters. An auxiliary MHA block is employed as the last layer of the MiniLM for better alignment following~\citet{Wang21}. The MiniLM can be then finetuned on any tasks.

\begin{figure}[t]
    \centering
    \includegraphics[width=0.47\textwidth]{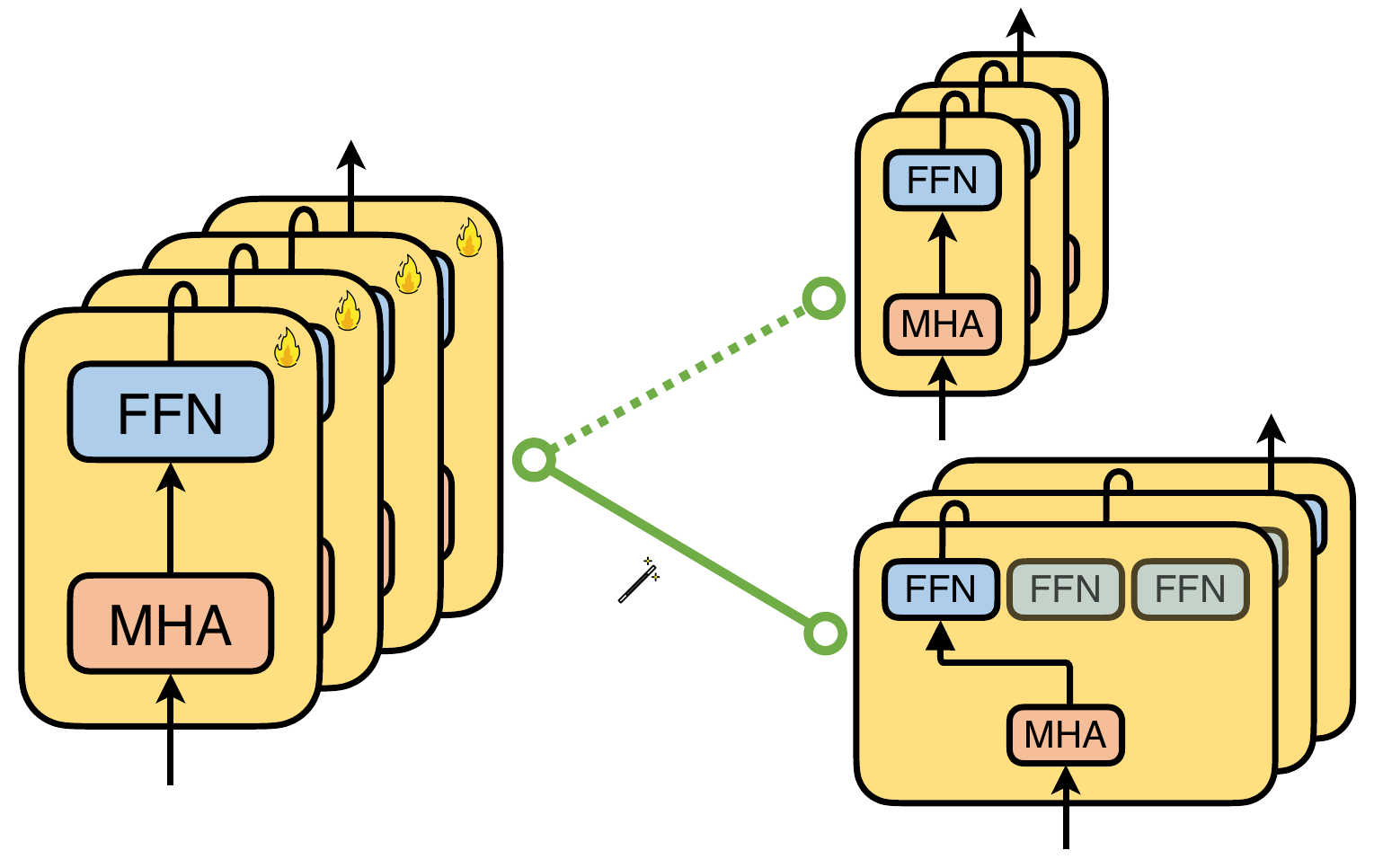}
    \caption{Implementation of \textsc{MiniMoE}.}
    \label{fig_x}
\end{figure}

\paragraph{Mixture of Minimal Experts}

Naturally, in order to enlarge the learning capacity gap of the student, we should add more parameters to the student. However, trivially adding parameters usually leads to a loss of inference compute efficiency. 

To remedy this, a mixture of minimal experts is proposed as in Figure~\ref{fig_x}. Following prior literature~\citep{ShazeerMMDLHD17,ShazeerCPTVKHLH18}, if we consider a FFN block in a MiniLM as a minimal expert, then extra parameters are exactly imposed as minimal experts to be added to the FFN block. The FFN block is enabled as a mixture of $m$ minimal experts FFN$^{\sf MoE}$ in an expert gating tactic as:
\begin{equation}\nonumber
    \text{FFN}^{\sf MoE}(\mathbf{x}_{i})=p_{k}(\mathbf{x}_{i})\cdot\text{FFN}_{k}(\mathbf{x}_{i}),
\end{equation}
\begin{equation}\nonumber
    p_{k}(\mathbf{x}_{i})=\frac{\exp(\mathbf{x}_{i}\mathbf{w}^{\sf G}_{k})}{\sum^{m}_{j=1}\exp(\mathbf{x}_{i}\mathbf{w}^{\sf G}_{j})},
\end{equation}
\begin{equation}\nonumber
    k=\text{argmax}\ p(\mathbf{x}_{i}),
\end{equation}
where the $j$-th gate is parameterized by $\mathbf{w}^{\sf G}_{j}\in\mathbb{R}^{d}$, and correspondingly the $j$-th minimal expert is denoted as FFN$_{j}$. We further follow~\citet{Fedus21} to only allow top-\textit{one} gating (i.e., only the expert with highest gating probability is reserved) because we want to keep the inference compute untouched. There are also diverse designs to achieve the sparse routing, such as hashing~\citep{RollerSSW21} which we find performs worse (cf. Figure~\ref{fig_5}). 

Since only one minimal expert is activated during the inference, the compute is only negligibly increased by expert routing. As a complement, we can also achieve, if necessary, a mixture of experts in an MHA block similarly. 

To encourage a balanced load across minimal experts, a differentiable load balancing objective $\mathcal{B}(\mathcal{S};\mathcal{D})$ is added from~\citet{LepikhinLXCFHKS21} as:
\begin{equation}\nonumber
    \mathcal{B}(\mathcal{S};\mathcal{D})=\alpha\cdot m\sum^{m}_{j=1}f_{j}\cdot P_{j},
\end{equation}
\begin{equation}\nonumber
    f_{j}=\mathbb{E}_{\mathbf{x}_{i}\sim\mathcal{D}}[\mathbb{I}\{\text{argmax}\ p(\mathbf{x}_{i}),j\}], 
\end{equation}
\begin{equation}\nonumber
    P_{j}=\mathbb{E}_{\mathbf{x}_{i}\sim\mathcal{D}}[p_{j}(\mathbf{x}_{i})],
\end{equation}
where $\alpha$ is a coefficient that should be manually tune and is kept as 0.01 throughout this work following~\citep{Fedus21}. While $f_{j}$ depicts the fraction of tokens dispatched to the $j$-th minimal expert, $P_{j}$ describes the fraction of the routing probability to the $j$-th minimal expert. And a multiplier $m$ is used to make the magnitude of the objective invariant to the number of minimal experts. The load balancing objective basically desires a uniform routing so that the loss can be minimized. The objective is added to the MiniLM not only at task-agnostic distillation stage but also at finetuning stage for practical concerns (cf. Figure~\ref{fig_5}).

\section{Experiments}

\subsection{Data and Metrics}

We conduct experiments on GLUE~\citep{WangSMHLB19} and CoNLL~\citep{SangM03}. The GLUE originally consists of two sequence classification tasks, SST-2~\citep{SocherPWCMNP13}, i.e., CoLA~\citep{WarstadtSB19}, and seven sequence-pair classification tasks, i.e., MRPC~\citep{DolanB05}, STS-B~\citep{CerDALS17}, QQP, MNLI~\citep{WilliamsNB18}, QNLI~\citep{RajpurkarZLL16}, RTE~\citep{BentivogliCDG11}, WNLI~\citep{LevesqueDM12}. We exclude WNLI and CoLA due to the evaluation inconsistency (in other words, MiniLMs get dramatically worse results while LMs get much better ones as found out in~\citealp{XiaZC22}) and use the left tasks. The CoNLL is a token classification task. Following BERT~\citep{DevlinCLT19}, we report Accuracy (Acc) on SST-2, MNLI, QNLI, RTE, Spearman Correlation scores (SpCorr) on STS-B, and F1 on MRPC, QQP, CoNLL. Average score over tasks from GLUE (GLUE Score) is additionally computed. Results on development sets are reported. GFLOPs are also attached as theoretical speedup references. We adopt Wikipedia data for task-agnostic disitllation. The detailed statistics, maximum sequence lengths, and metrics of GLUE, CoNLL, and Wikipedia are supplied in Appendix~\ref{app_a}. 

\subsection{Hands-on Details}

Experiments are conducted upon distilling BERT\textsubscript{\sf base} and BERT\textsubscript{\sf large}~\citep{DevlinCLT19}. The distillation carried out on eight Nvidia A100s. The number of relation heads is set to 32. 
After the distillation, finetuning is carried out on one Nvidia A100. The number of minimal experts $m$ is default to 4 otherwise specified.
Other details are supplied in Appendix~\ref{app_b}. All experiments are task-agnostic ones, except those in Table~\ref{tab_6}.

\begin{table*}[ht]
    \caption{The results of comparison between distilling BERT\textsubscript{\sf base} and BERT\textsubscript{\sf large}.}
    \begin{adjustbox}{width=\textwidth,center}
    \begin{threeparttable}
    \begin{tabular}{ll|ccccccc|c|c}
    \toprule
      \textbf{Method} & \textbf{Teacher} & \makecell[c]{\textbf{SST-2}\\\textbf{Acc}} & \makecell[c]{\textbf{MRPC}\\\textbf{F1}} & \makecell[c]{\textbf{STS-B}\\\textbf{SpCorr}} & \makecell[c]{\textbf{QQP}\\\textbf{F1}} & \makecell[c]{\textbf{MNLI-m/mm}\\\textbf{Acc}} & \makecell[c]{\textbf{QNLI}\\\textbf{Acc}} & \makecell[c]{\textbf{RTE}\\\textbf{Acc}} & \makecell[c]{\textbf{GLUE}\\\textbf{Score}} & \makecell[c]{\textbf{CoNLL}\\\textbf{F1}} \\
    \midrule
      \multirow{2}{*}{MiniLM\textsubscript{\sf 6L;384H}} & BERT\textsubscript{\sf base} & 91.1 & 90.1 & 88.1 & 86.7 & 81.5/81.8 & 89.2 & 67.9 & 84.5 & 93.2 \\
      & BERT\textsubscript{\sf large}$\textcolor[rgb]{0,0.7,0}{\boldsymbol{\Uparrow}}$ & 90.9 & 90.6 & 89.0 & 86.9 & 81.8/82.4 & 88.8 & 70.0 & 85.1 & 93.2 \\
      \cmidrule{2-11}
      \multirow{2}{*}{\quad w/ TA} & BERT\textsubscript{\sf base} & 91.3 & 90.3 & 88.2 & 86.8 & 81.4/81.6 & 89.7 & 66.8 & 84.5 & 93.2 \\
      & BERT\textsubscript{\sf large}$\textcolor[rgb]{0,0.7,0}{\boldsymbol{\Uparrow}}$ & 91.4 & 89.8 & 88.5 & 87.0 & 81.9/81.6 & 89.5 & 71.5 & 85.2 & 93.2 \\
      \cmidrule{2-11}
      \rowcolor{green!20} & BERT\textsubscript{\sf base} & 91.3 & 90.2 & 88.6 & 86.5 & 81.6/81.5 & 89.5 & 68.6 & 84.7 & 93.3 \\
      \rowcolor{green!20} \multirow{-2}{*}{ \textsc{MiniMoE}\textsubscript{\sf 6L;384H}} & BERT\textsubscript{\sf large}$\textcolor[rgb]{0,0.7,0}{\boldsymbol{\Uparrow}}$\tnote{1} & 90.5 & 90.0 & 88.8 & 86.8 & 81.8/82.2 & 90.8 & 70.4 & 85.2 & 93.3 \\
    \midrule
      \multirow{2}{*}{MiniLM\textsubscript{\sf 4L;384H}} & BERT\textsubscript{\sf base} & 90.0 & 88.6 & 87.2 & 86.1 & 80.0/80.3 & 87.9 & 67.2 & 83.4 & 91.5 \\
      & BERT\textsubscript{\sf large}$\textcolor[rgb]{1,0,0}{\boldsymbol{\Downarrow}}$ & 89.3 & 87.5 & 88.1 & 85.9 & 79.9/80.2 & 87.6 & 67.2 & 83.2 & 91.2 \\
      \cmidrule{2-11}
      \multirow{2}{*}{\quad w/ TA} & BERT\textsubscript{\sf base} & 90.0 & 88.5 & 87.3 & 86.3 & 80.1/80.7 & 88.0 & 66.4 & 83.4 & 91.8 \\
      & BERT\textsubscript{\sf large}$\textcolor[rgb]{0,0.7,0}{\boldsymbol{\Uparrow}}$ & 90.6 & 88.7 & 88.1 & 86.3 & 80.5/80.7 & 87.9 & 69.0 & 84.0 & 92.2 \\
      \cmidrule{2-11}
      \rowcolor{green!20} & BERT\textsubscript{\sf base} & 90.8 & 88.1 & 88.2 & 85.9 & 79.8/80.4 & 88.6 & 69.3 & 83.9 & 92.3 \\
      \rowcolor{green!20} \multirow{-2}{*}{\textsc{MiniMoE}\textsubscript{\sf 4L;384H}} & BERT\textsubscript{\sf large}$\textcolor[rgb]{0,0.7,0}{\boldsymbol{\Uparrow}}$ & 90.5 & 88.0 & 88.7 & 86.7 & 80.9/80.9 & 89.2 & 69.0 & 84.2 & 92.4 \\
    \midrule
      \multirow{2}{*}{MiniLM\textsubscript{\sf 3L;384H}} & BERT\textsubscript{\sf base} & 89.1 & 89.1 & 86.6 & 85.4 & 77.8/78.4 & 87.2 & 66.1 & 82.5 & 90.1 \\
      & BERT\textsubscript{\sf large}$\textcolor[rgb]{1,0,0}{\boldsymbol{\Downarrow}}$ & 89.1 & 86.1 & 87.1 & 85.1 & 78.6/78.5 & 86.0 & 65.7 & 82.0 & 87.3 \\
      \cmidrule{2-11}
      \multirow{2}{*}{\quad w/ TA} & BERT\textsubscript{\sf base} & 89.8 & 87.8 & 86.0 & 85.5 & 77.6/78.5 & 86.8 & 66.1 & 82.3 & 90.4 \\
      & BERT\textsubscript{\sf large}$\textcolor[rgb]{1,0,0}{\boldsymbol{\Downarrow}}$ & 89.7 & 84.9 & 87.2 & 85.2 & 78.5/79.1 & 86.6 & 66.4 & 82.2 & 90.2 \\
      \cmidrule{2-11}
      \rowcolor{green!20} & BERT\textsubscript{\sf base} & 89.3 & 87.4 & 87.8 & 85.6 & 78.2/78.7 & 87.2 & 67.0 & 82.6 & 90.7 \\
      \rowcolor{green!20} \multirow{-2}{*}{\textsc{MiniMoE}\textsubscript{\sf 3L;384H}} & BERT\textsubscript{\sf large}$\textcolor[rgb]{0,0.7,0}{\boldsymbol{\Uparrow}}$ & 89.1 & 88.4 & 87.6 & 86.2 & 78.8/79.5 & 87.5 & 67.9 & 83.1 & 91.6 \\
    \bottomrule
    \end{tabular}
    \begin{tablenotes}
      \item [1] $\textcolor[rgb]{0,0.7,0}{\boldsymbol{\Uparrow}}$ is used to indicate the deficiency is tackled on both GLUE and CoNLL, otherwise $\textcolor[rgb]{1,0,0}{\boldsymbol{\Downarrow}}$ is used.
      
     \end{tablenotes}
    \end{threeparttable}
    \end{adjustbox}
    \label{tab_4}
\end{table*}

\subsection{Baselines}

We compare \textsc{MiniMoE} with several state-of-the-art baselines.

\paragraph{Conventional Distillation} 

FT indicates direct finetuning the student. KD~\citep{HintonVD15}, PKD~\citep{SunCGL19}, and CKD~\citep{ParkKY21} are methods with different distillation objectives, i.e., KD directly distills logits, PKD distills both logits and hidden states, and CKD distills high-order relations. While above four methods originally initialize student structures by dropping layers, we enable them with a global pruning so that they can adapt to students of small scales. DynaBERT~\citep{HouHSJCL20} uses a two-step pruning to regulate student structures and a distillation objective akin to PKD. MoEBERT~\citep{ZuoZLHZC22} moefies LMs by decomposing FFN blocks to MoE layers. For these task-specific distillation methods, student structures are denoted either with \textsubscript{\sf *L} for preserved number of layers in layer-dropping or with \textsubscript{\sf *\%} for preserved portion of parameters in pruning.

As aforementioned methods are task-specific distillation ones, we then introduce task-agnostic ones. TinyBERT~\citep{JiaoYSJCL0L20} exploits a distillation objective distilled with a combination of various feature alignments. MiniLM~\citep{WangBHDW21} straightforwardly utilizes a distillation objective with a deep relation alignment exactly the same with ours. Since task-agnostic distillation allows both dropping layers and hidden dimensions, student structures are denoted with \textsubscript{\sf *L;*H} accordingly.

\paragraph{Capacity-aware Distillation} 

MiniLM w/ TA~\citep{WangW0B0020} specifically incorporates a teacher assistant to MiniLM. MiniDisc~\citep{Zhang22} argues that the scale of the teacher assistant is crucial for student performance and proposes an automatic teacher assistant scheduler based on properties of pruning. While MiniLM w/ TA is only inspected under a task-agnostic setting, MiniDisc offers results under both task-specific and task-agnostic settings. Nevertheless, only task-specific MiniDisc is selected since pruned MiniLMs can be unfair to compare with. There is scarce work in this direction in which we find these two are the most comparable ones.

\begin{table*}[ht]
    \caption{The results of comparison between \textsc{MiniMoE} and baselines upon distilling BERT\textsubscript{\sf base}. The best results are \textbf{boldfaced}.}
    \begin{adjustbox}{width=\textwidth,center}
    \begin{threeparttable}
    \begin{tabular}{lll|ccccccc|c|c}
    \toprule
      \textbf{Method} & \multicolumn{2}{l|}{\textbf{GFLOPs}} & \makecell[c]{\textbf{SST-2}\\\textbf{Acc}} & \makecell[c]{\textbf{MRPC}\\\textbf{F1}} & \makecell[c]{\textbf{STS-B}\\\textbf{SpCorr}} & \makecell[c]{\textbf{QQP}\\\textbf{F1}} & \makecell[c]{\textbf{MNLI-m/mm}\\\textbf{Acc}} & \makecell[c]{\textbf{QNLI}\\\textbf{Acc}} & \makecell[c]{\textbf{RTE}\\\textbf{Acc}} & \makecell[c]{\textbf{GLUE}\\\textbf{Score}} & \makecell[c]{\textbf{CoNLL}\\\textbf{F1}} \\
    \midrule
      BERT\textsubscript{\sf base} & 10.9 & & 93.8 & 91.5 & 87.1 & 88.4 & 84.9/84.9 & 91.9 & 71.5 & 86.7 & 94.8 \\
    \midrule
      KD\textsubscript{\sf 15\%} & 1.64 & & 89.9 & 88.6 & 85.1 & 86.2 & 79.8/80.2 & 85.6 & 63.9 & 82.4 & 92.8 \\
      PKD\textsubscript{\sf 15\%} & 1.64 & & 90.0 & 88.2 & 85.5 & 86.4 & 80.4/79.6 & 85.9 & 63.9 & 82.5 & 92.9 \\
      MoEBERT\textsubscript{\sf 17\%}\tnote{1} & 1.86 & & 89.6 & 88.4 & 85.1 & 86.8 & 80.4/80.5 & 86.6 & 65.0 & 82.8 & 92.7 \\
      DynaBERT\textsubscript{\sf 15\%}\tnote{2} & 1.64 & & 89.1 & 85.1 & 84.7 & 84.3 & 78.3/79.0 & 86.6 & 61.4 & 81.1 & - \\
      MiniDisc\textsubscript{\sf 15\%}\tnote{3} & 1.64 & & 89.8 & 88.2 & 85.8 & 86.6 & 80.3/79.9 & 87.3 & 68.2 & 83.3 & 93.0 \\
      MiniLM\textsubscript{\sf 6L;384H} & 1.36 & & 91.1 & 90.1 & 88.1 & 86.7 & 81.5/\textbf{81.8} & 89.2 & 67.9 & 84.5 & 93.2 \\
      \quad w/ TA & 1.36 & & 91.3 & \textbf{90.3} & 88.2 & \textbf{86.8} & 81.4/81.6 & \textbf{89.7} & 66.8 & 84.5 & 93.2 \\
      \rowcolor{green!20} \textsc{MiniMoE}\textsubscript{\sf 6L;384H} & 1.36 &  \multirow{-8}{*}{\rotatebox[origin=c]{90}{6$\sim$8$\times$}} & \textbf{91.3} & 90.2 & \textbf{88.6} & 86.5 & \textbf{81.6}/81.5 & 89.5 & \textbf{68.6} & \textbf{84.7} & \textbf{93.3} \\
    \midrule
      KD\textsubscript{\sf 10\%} & 1.08 & & 88.2 & 87.6 & 84.0 & 84.4 & 77.6/77.4 & 84.3 & 67.2 & 81.3 & 91.2 \\
      MiniDisc\textsubscript{\sf 10\%} & 1.08 & & 89.1 & 88.4 & 85.4 & 84.9 & 78.2/78.6 & 86.3 & 68.2 & 82.4 & 91.9 \\
      MiniLM\textsubscript{\sf 4L;384H} & 0.91 & & 90.0 & \textbf{88.6} & 87.2 & 86.1 & 80.0/80.3 & 87.9 & 67.2 & 83.4 & 91.5 \\
      \quad w/ TA & 0.91 & & 90.0 & 88.5 & 87.3 & \textbf{86.3} & \textbf{80.1}/\textbf{80.7} & 88.0 & 66.4 & 83.4 & 91.8 \\
      \rowcolor{green!20} \textsc{MiniMoE}\textsubscript{\sf 4L;384H} & 0.91 &  \multirow{-5}{*}{\rotatebox[origin=c]{90}{10$\sim$12$\times$}} & \textbf{90.8} & 88.1 & \textbf{88.2} & 85.9 & 79.8/80.4 & \textbf{88.6} & \textbf{69.3} & \textbf{83.9} & \textbf{92.3} \\
    \midrule
      KD\textsubscript{\sf 5\%} & 0.54 & & 85.6 & 84.0 & 83.8 & 82.5 & 72.6/73.2 & 81.6 & 63.2 & 78.3 & 83.1 \\
      MiniDisc\textsubscript{\sf 5\%} & 0.54 & & 86.9 & 87.6 & 84.8 & 83.5 & 72.7/74.5 & 84.0 & 66.8 & 80.1 & 85.6 \\
      TinyBERT\textsubscript{\sf 4L;312H}\tnote{4} & 0.60 & & 88.5 & 87.9 & 86.6 & 85.6 & \textbf{78.9}/\textbf{79.2} & \textbf{87.3} & \textbf{67.2} & \textbf{82.7} & - \\
      MiniLM\textsubscript{\sf 3L;384H} & 0.68 & & 89.1 & \textbf{89.1} & 86.6 & 85.4 & 77.8/78.4 & 87.2 & 66.1 & 82.5 & 90.1 \\
      \quad w/ TA & 0.68 & & \textbf{89.8} & 87.8 & 86.0 & 85.5 & 77.6/78.5 & 86.8 & 66.1 & 82.3 & 90.4 \\
      \rowcolor{green!20} \textsc{MiniMoE}\textsubscript{\sf 3L;384H} & 0.68 &  \multirow{-6}{*}{\rotatebox[origin=c]{90}{16$\sim$20$\times$}} & 89.3 & 87.4 & \textbf{87.8} & \textbf{85.6} & 78.2/78.7 & 87.2 & 67.0 & 82.6 & \textbf{90.7} \\
    \midrule
      KD\textsubscript{\sf 3\%} & 0.32 & & 85.2 & 83.6 & 81.9 & 82.1 & 71.9/72.7 & 81.9 & 57.4 & 77.1 & 74.3 \\
      MiniDisc\textsubscript{\sf 3\%} & 0.32 & & 85.9 & 85.7 & 83.6 & 83.1 & 72.9/73.6 & 81.9 & 58.1 & 78.1 & 80.5 \\
      MiniLM\textsubscript{\sf 4L;192H} & 0.23 & & 86.9 & \textbf{86.4} & 85.4 & 84.3 & 77.5/77.5 & 85.9 & 65.3 & 81.2 & 90.0 \\
      \quad w/ TA & 0.23 & & 87.2 & 85.6 & 86.2 & 84.6 & 77.3/\textbf{78.0} & 86.6 & 64.6 & 81.3 & 89.9 \\
      \rowcolor{green!20} \textsc{MiniMoE}\textsubscript{\sf 4L;192H} & 0.23 &  \multirow{-5}{*}{\rotatebox[origin=c]{90}{34$\sim$47$\times$}} & \textbf{88.1} & 86.1 & \textbf{86.2} & \textbf{84.8} & \textbf{77.7}/77.8 & \textbf{86.6} & \textbf{68.6} & \textbf{82.0} & \textbf{91.3} \\
    \bottomrule
    \end{tabular}
    \begin{tablenotes}
      \item [1] Each FFN is split to 8 experts and each MHA to 4 to reach the sparsity.
      \item [2] The results are produced from the released code.
      \item [3] The results are mainly taken from the original papers.
      \item [4] The results are produced without additional task-specific distillation. 
     \end{tablenotes}
    \end{threeparttable}
    \end{adjustbox}
    \label{tab_3}
\end{table*}

\subsection{Main Results}

From results in Table~\ref{tab_4}, we find that \textsc{MiniMoE} lifts the curse of capacity gap at all concerned times of compression. For example, \textsc{MiniMoE}\textsubscript{\sf 3L;384H} disitlled from BERT\textsubscript{\sf large} has an absolute 0.5 performance gain over that distilled from BERT\textsubscript{\sf base} on GLUE, and the value on CoNLL is 0.9. On another note, MiniLM is free of the curse only at small times of compression, and MiniLM w/ TA can somewhat saves MiniLM from the curse at intermediate times of compression. For example, both MiniLM\textsubscript{\sf 3L;384H} and MiniLM\textsubscript{\sf 3L;384H} w/ TA fail to improve the performance via replacing BERT\textsubscript{\sf base} with BERT\textsubscript{\sf large}. Results on larger LMs like BERT\textsubscript{\sf xlarge} are supplied in Appendix~\ref{app_f} for scalability check.

From results in Table~\ref{tab_3}, we also observe that \textsc{MiniMoE} generally outperforms both conventional and capacity-aware baselines and achieves new state-of-the-art performance at all concerned times of compression. For example, \textsc{MiniMoE}\textsubscript{\sf 4L;192H} has an absolute 0.8 performance improvement over MiniLM\textsubscript{\sf 4L;192H} on GLUE. And the reason why \textsc{MiniMoE}\textsubscript{\sf 3L;384H} slightly underperforms TinyBERT\textsubscript{\sf 4L;312H} is conjectured due to structure discrepancy. Another observation is that the larger times of compression, the larger the performance improvements are. For example, \textsc{MiniMoE}\textsubscript{\sf 4L;384H} yields an absolute 0.5 performance improvement over MiniLM\textsubscript{\sf 4L;384H} in contrast to that \textsc{MiniMoE}\textsubscript{\sf 6L;384H} only has an absolute 0.2 performance improvement over MiniLM\textsubscript{\sf 6L;384H} on GLUE. Two more notes are that, MoEBERT nearly reaches the compression upper bound, and TinyBERT is reproduced without additional task-specific distillation for a fair comparison while the results with additional task-specific distillation are supplied in Appendix~\ref{app_c}.

\subsection{Analyses}

\paragraph{Practical Inference Compute}

Since GFLOPs can only measure the theoretical inference compute, we further provide throughput (i.e., tokens per micro second) as a practical inference compute measure. As in Table~\ref{tab_5}, 20$\times$ compression can realize a significant inference compute gain in comparing KD\textsubscript{\sf 5\%} to BERT\textsubscript{\sf base}. The practical speedup is approximately 6.7$\times$. Moreover, \textsc{MiniMoE}\textsubscript{\sf 3L;384H} can retain most inference compute gain even if the routing algorithm can slightly reduce the gain when compared to MiniLM\textsubscript{\sf 3L;384H}. Although \textsc{MiniMoE} is seemingly memory-inefficient regarding the increased parameter amount, we argue the potential of a memory-efficient \textsc{MiniMoE} with parameter decomposition in Appendix~\ref{app_g}.

\begin{table}[ht]
    \centering
    \caption{Practical inference compute with reference to BERT\textsubscript{\sf base}.}
    \begin{adjustbox}{width=0.50\textwidth,center}
    \begin{tabular}{llrr}
    \toprule
      \textbf{Method} & \textbf{GFLOPs} & \textbf{Throughput} & \textbf{Params} \\
    \midrule
      BERT\textsubscript{\sf base} & 10.9 & 80.8 tokens/ms & 109.5 M \\
      KD\textsubscript{\sf 5\%} & 0.54 & 544.7 tokens/ms & 28.7 M \\
      MiniLM\textsubscript{\sf 3L;384H} & 0.68 & 485.3 tokens/ms & 17.2 M\\
      \textsc{MiniMoE}\textsubscript{\sf 3L;384H} & 0.68 & 433.1 tokens/ms & 28.3 M \\
    \bottomrule
    \end{tabular}
    \end{adjustbox}
    \label{tab_5}
\end{table}

\paragraph{Student Scale}

Following the behavior of Figure~\ref{fig_2}, we would like to showcase whether \textsc{MiniMoE} can lift the curse across difference student scales. From Figure~\ref{fig_3}, the curse is lifted to a large extent by \textsc{MiniMoE} in comparison with MiniLM and MiniLM w/ TA. However, \textsc{MiniMoE} meets a bottleneck that distilling BERT\textsubscript{\sf large} makes no difference from distilling BERT\textsubscript{\sf base} when the FLOPs is at an extreme value 0.04G ($\sim$273$\times$ compression from BERT\textsubscript{\sf base}, $\sim$968$\times$ compression from BERT\textsubscript{\sf large}). We explore the extreme case by plugging a TA to \textsc{MiniMoE} as supplied in Appendix~\ref{app_d}.

\begin{figure}[ht]
    \centering
    \includegraphics[width=0.47\textwidth]{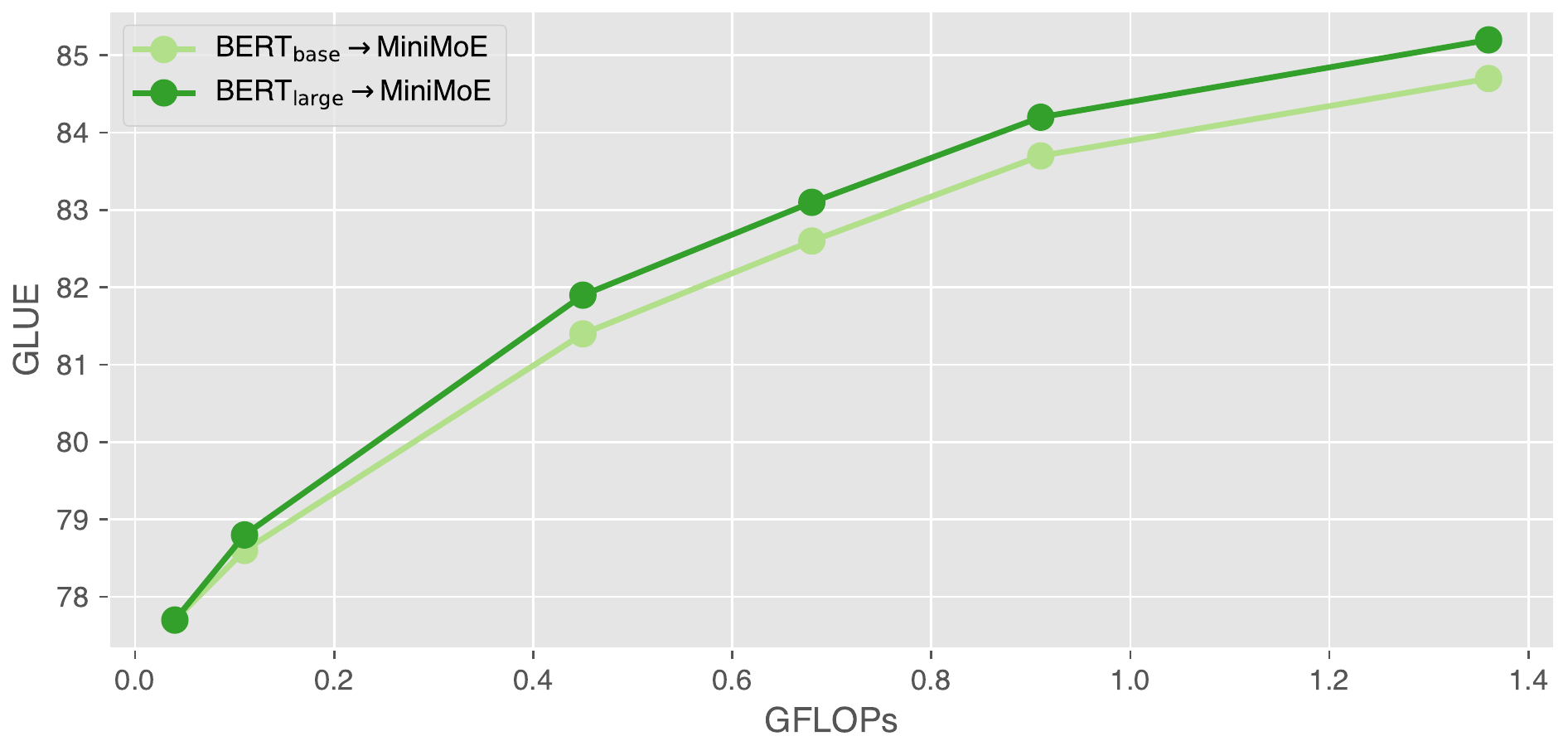}
    \caption{The performance of \textsc{MiniMoE} across different student scales upon distilling BERT\textsubscript{\sf base}.}
    \label{fig_3}
\end{figure}

\begin{figure}[ht]
    \centering
    \includegraphics[width=0.47\textwidth]{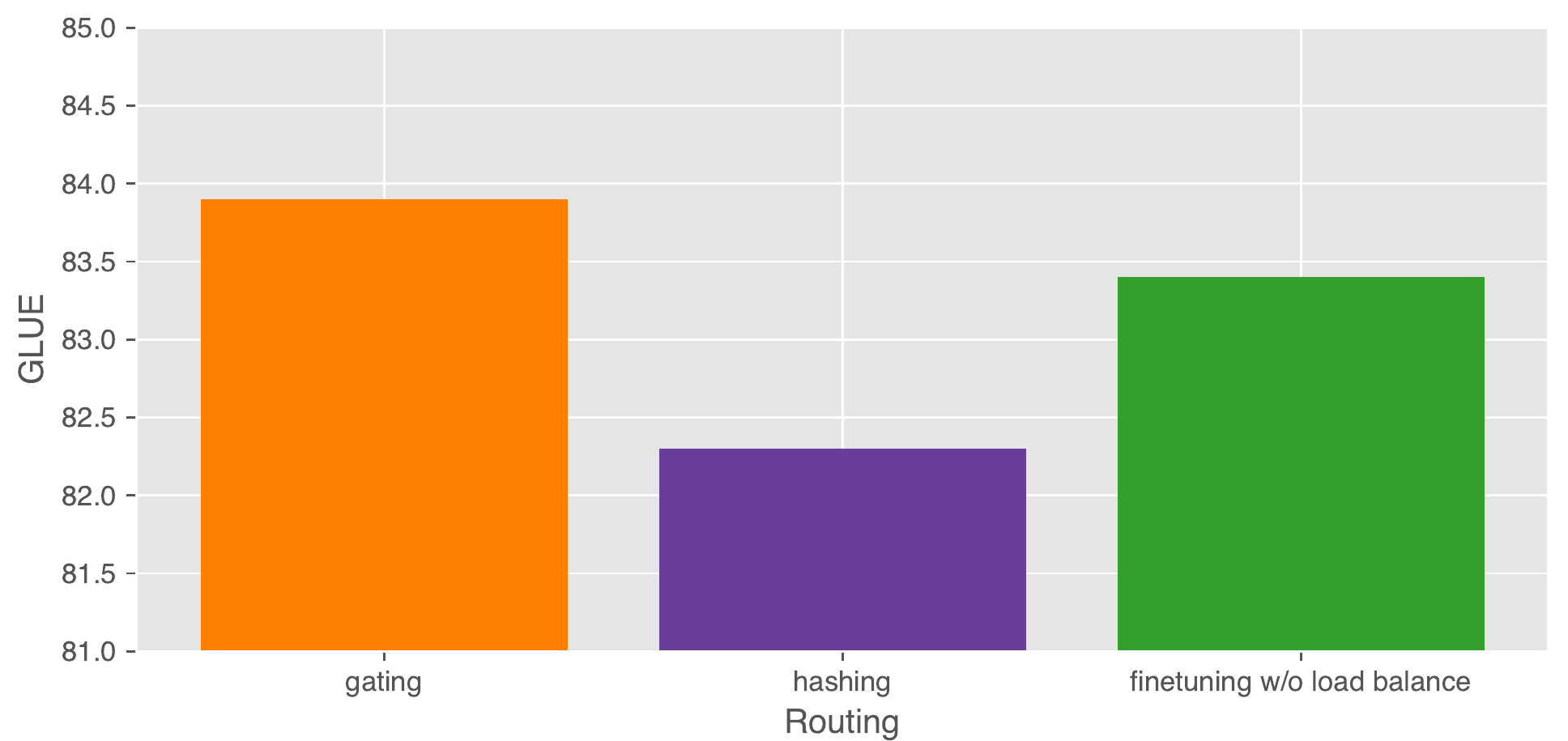}
    \caption{The performance of different routing choices with MiniMoE\textsubscript{\sf 4L;384H} upon distilling BERT\textsubscript{\sf base}.}
    \label{fig_5}
\end{figure}


\paragraph{Routing Algorithm}

Routing algorithm is also a crucial part benefiting from a nice design choice. We compare our used gating with another fancy choice hashing. We at the same time show the effect of using load balance at finetuning stage as well. From the results in Figure~\ref{fig_5}, we see that gating outperforms hashing, and load balancing at both distillation and finetuning stages is superior to that at only distillation stage. 

\paragraph{Expert Number}

Regarding the expert number $m$ is a parameter of great importance for \textsc{MiniMoE}, we here study its impact on the performance. The results in Figure~\ref{fig_4} reveal a first ascending then descending phenomenon while adding experts at a time. The phenomenon suggests there is a tradeoff when increasing the number of experts, and we conjecture the tradeoff accords with the famous bias-variance tradeoff~\citep[][Chapter 7]{HastieFT01}. That is, adding experts grows the parameter scale, thus decreasing bias yet increasing variance. Another interesting notice is that smaller students favor fewer experts. Based on the tradeoff conjecture, we hypothesize that smaller students are more sensitive to variance increment, as the biases of smaller students can arrive at a minimum more quickly than those of larger ones.

\begin{figure}[ht]
    \centering
    \includegraphics[width=0.47\textwidth]{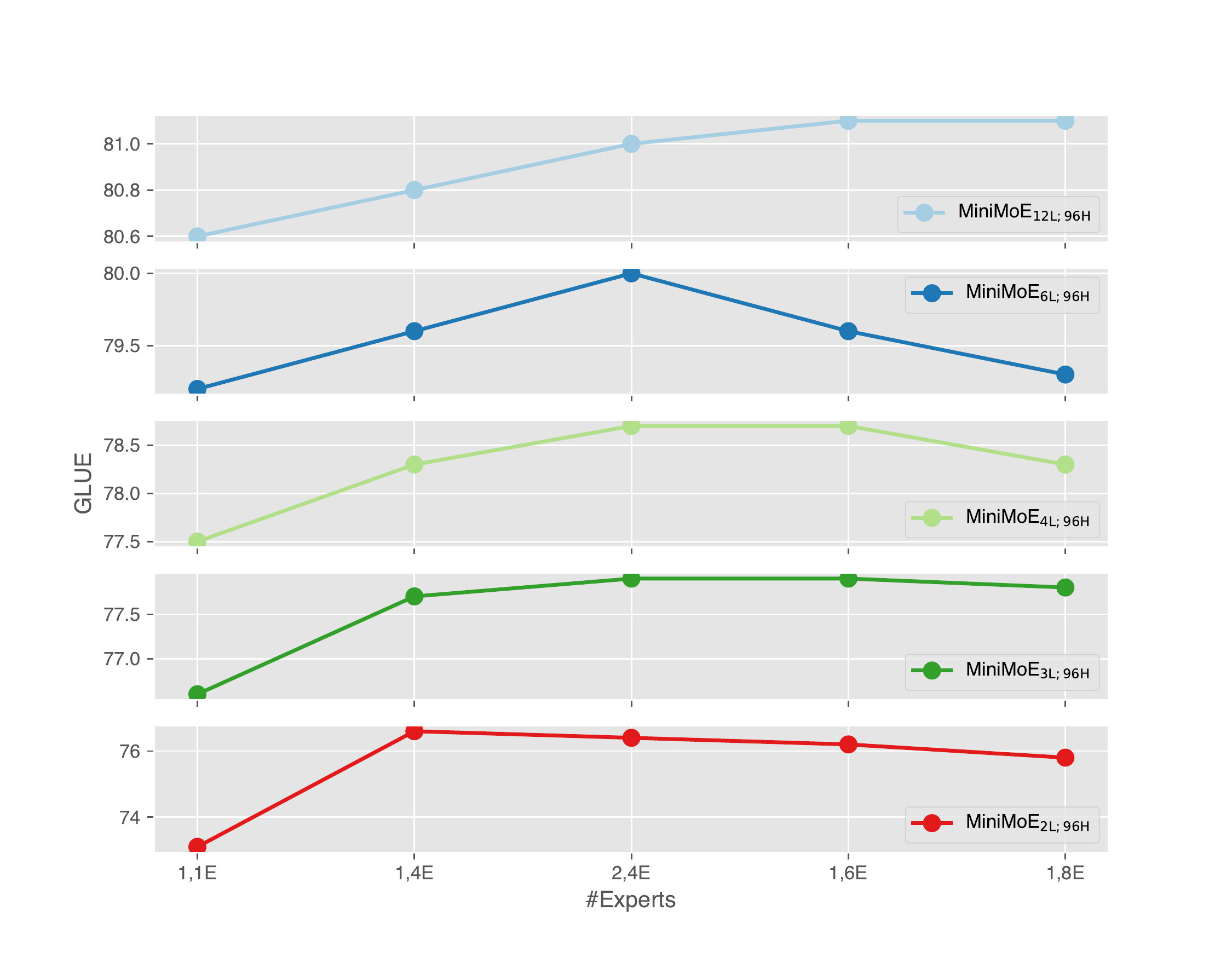}
    \caption{The impact of expert number on the performance upon distilling BERT\textsubscript{\sf base}, where {\sf x,yE} denotes {\sf x} experts in each MHA and {\sf y} experts in each FFN. For example, {\sf 1,1E} is the original dense model, and {\sf 1,4E} is the MoE model used in Table~\ref{tab_3}.}
    \label{fig_4}
\end{figure}

\section{Conclusions}

In this work, we uncover a curse of capacity gap in LM distillation, which is well discussed in previous studies on vision model distillation but not recognized in distilling LMs. While there are some studies investigating to fill the gap, we find they can hardly tackle the curse. Interestingly, existing solutions in large vision language model distillation which are stated to be able to lift the curse fail to achieve so for LMs. So we aim at lifting the curse by proposing a well-motivated \textsc{MiniMoE}. The \textsc{MiniMoE} can essentially enlarge the capacity of the student but leave the inference compute nearly untouched. Our experimental results indicate that \textsc{MiniMoE} can not only lift the curse but also realize new state of the arts.

\section*{Limitations}

The central limitation of \textsc{MiniMoE} is the increased memory footprint, which we could potentially address in the near future according to Appendix~\ref{app_g}.


\section*{Acknowledgements}

We thank the anonymous reviewers and chairs for their constructive suggestions. This research was supported in part by Natural Science Foundation of Beijing (grant number: 4222036) and Huawei Technologies (grant number: TC20201228005). Jingang Wang is funded by Beijing Nova Program (grant number: 20220484098).

\bibliography{anthology,custom}

\begin{thebibliography}{66}
\expandafter\ifx\csname natexlab\endcsname\relax\def\natexlab#1{#1}\fi

\bibitem[{Bai et~al.(2021)Bai, Zhang, Hou, Shang, Jin, Jiang, Liu, Lyu, and
  King}]{BaiZHSJJLLK20}
Haoli Bai, Wei Zhang, Lu~Hou, Lifeng Shang, Jin Jin, Xin Jiang, Qun Liu,
  Michael~R. Lyu, and Irwin King. 2021.
\newblock \href {https://doi.org/10.18653/v1/2021.acl-long.334} {Binarybert:
  Pushing the limit of {BERT} quantization}.
\newblock In \emph{Proceedings of the 59th Annual Meeting of the Association
  for Computational Linguistics and the 11th International Joint Conference on
  Natural Language Processing, {ACL/IJCNLP} 2021, (Volume 1: Long Papers),
  Virtual Event, August 1-6, 2021}, pages 4334--4348.

\bibitem[{Bengio et~al.(2015)Bengio, Bacon, Pineau, and Precup}]{BengioBPP15}
Emmanuel Bengio, Pierre{-}Luc Bacon, Joelle Pineau, and Doina Precup. 2015.
\newblock \href {http://arxiv.org/abs/1511.06297} {Conditional computation in
  neural networks for faster models}.
\newblock \emph{arXiv}, 1511.06297.

\bibitem[{Bentivogli et~al.(2011)Bentivogli, Clark, Dagan, and
  Giampiccolo}]{BentivogliCDG11}
Luisa Bentivogli, Peter Clark, Ido Dagan, and Danilo Giampiccolo. 2011.
\newblock \href
  {https://tac.nist.gov/publications/2011/additional.papers/RTE7\_overview.proceedings.pdf}
  {The seventh {PASCAL} recognizing textual entailment challenge}.
\newblock In \emph{Proceedings of the Fourth Text Analysis Conference, {TAC}
  2011, Gaithersburg, Maryland, USA, November 14-15, 2011}.

\bibitem[{Bucila et~al.(2006)Bucila, Caruana, and
  Niculescu{-}Mizil}]{BucilaCN06}
Cristian Bucila, Rich Caruana, and Alexandru Niculescu{-}Mizil. 2006.
\newblock \href {https://doi.org/10.1145/1150402.1150464} {Model compression}.
\newblock In \emph{Proceedings of the Twelfth {ACM} {SIGKDD} International
  Conference on Knowledge Discovery and Data Mining, Philadelphia, PA, USA,
  August 20-23, 2006}, pages 535--541.

\bibitem[{Cer et~al.(2017)Cer, Diab, Agirre, Lopez{-}Gazpio, and
  Specia}]{CerDALS17}
Daniel~M. Cer, Mona~T. Diab, Eneko Agirre, I{\~{n}}igo Lopez{-}Gazpio, and
  Lucia Specia. 2017.
\newblock \href {https://doi.org/10.18653/v1/S17-2001} {Semeval-2017 task 1:
  Semantic textual similarity multilingual and crosslingual focused
  evaluation}.
\newblock In \emph{Proceedings of the 11th International Workshop on Semantic
  Evaluation, SemEval@ACL 2017, Vancouver, Canada, August 3-4, 2017}, pages
  1--14.

\bibitem[{Cho and Hariharan(2019)}]{ChoH19}
Jang~Hyun Cho and Bharath Hariharan. 2019.
\newblock \href {https://doi.org/10.1109/ICCV.2019.00489} {On the efficacy of
  knowledge distillation}.
\newblock In \emph{2019 {IEEE/CVF} International Conference on Computer Vision,
  {ICCV} 2019, Seoul, Korea (South), October 27 - November 2, 2019}, pages
  4793--4801.

\bibitem[{Chowdhery et~al.(2022)Chowdhery, Narang, Devlin, Bosma, Mishra,
  Roberts, Barham, Chung, Sutton, Gehrmann, Schuh, Shi, Tsvyashchenko, Maynez,
  Rao, Barnes, Tay, Shazeer, Prabhakaran, Reif, Du, Hutchinson, Pope, Bradbury,
  Austin, Isard, Gur{-}Ari, Yin, Duke, Levskaya, Ghemawat, Dev, Michalewski,
  Garcia, Misra, Robinson, Fedus, Zhou, Ippolito, Luan, Lim, Zoph, Spiridonov,
  Sepassi, Dohan, Agrawal, Omernick, Dai, Pillai, Pellat, Lewkowycz, Moreira,
  Child, Polozov, Lee, Zhou, Wang, Saeta, Diaz, Firat, Catasta, Wei,
  Meier{-}Hellstern, Eck, Dean, Petrov, and Fiedel}]{Chowd22}
Aakanksha Chowdhery, Sharan Narang, Jacob Devlin, Maarten Bosma, Gaurav Mishra,
  Adam Roberts, Paul Barham, Hyung~Won Chung, Charles Sutton, Sebastian
  Gehrmann, Parker Schuh, Kensen Shi, Sasha Tsvyashchenko, Joshua Maynez,
  Abhishek Rao, Parker Barnes, Yi~Tay, Noam Shazeer, Vinodkumar Prabhakaran,
  Emily Reif, Nan Du, Ben Hutchinson, Reiner Pope, James Bradbury, Jacob
  Austin, Michael Isard, Guy Gur{-}Ari, Pengcheng Yin, Toju Duke, Anselm
  Levskaya, Sanjay Ghemawat, Sunipa Dev, Henryk Michalewski, Xavier Garcia,
  Vedant Misra, Kevin Robinson, Liam Fedus, Denny Zhou, Daphne Ippolito, David
  Luan, Hyeontaek Lim, Barret Zoph, Alexander Spiridonov, Ryan Sepassi, David
  Dohan, Shivani Agrawal, Mark Omernick, Andrew~M. Dai,
  Thanumalayan~Sankaranarayana Pillai, Marie Pellat, Aitor Lewkowycz, Erica
  Moreira, Rewon Child, Oleksandr Polozov, Katherine Lee, Zongwei Zhou, Xuezhi
  Wang, Brennan Saeta, Mark Diaz, Orhan Firat, Michele Catasta, Jason Wei,
  Kathy Meier{-}Hellstern, Douglas Eck, Jeff Dean, Slav Petrov, and Noah
  Fiedel. 2022.
\newblock \href {https://doi.org/10.48550/arXiv.2204.02311} {Palm: Scaling
  language modeling with pathways}.
\newblock \emph{arXiv}, abs/2204.02311.

\bibitem[{Chung et~al.(2022)Chung, Hou, Longpre, Zoph, Tay, Fedus, Li, Wang,
  Dehghani, Brahma, Webson, Gu, Dai, Suzgun, Chen, Chowdhery, Narang, Mishra,
  Yu, Zhao, Huang, Dai, Yu, Petrov, Chi, Dean, Devlin, Roberts, Zhou, Le, and
  Wei}]{Chung22}
Hyung~Won Chung, Le~Hou, Shayne Longpre, Barret Zoph, Yi~Tay, William Fedus,
  Eric Li, Xuezhi Wang, Mostafa Dehghani, Siddhartha Brahma, Albert Webson,
  Shixiang~Shane Gu, Zhuyun Dai, Mirac Suzgun, Xinyun Chen, Aakanksha
  Chowdhery, Sharan Narang, Gaurav Mishra, Adams Yu, Vincent~Y. Zhao, Yanping
  Huang, Andrew~M. Dai, Hongkun Yu, Slav Petrov, Ed~H. Chi, Jeff Dean, Jacob
  Devlin, Adam Roberts, Denny Zhou, Quoc~V. Le, and Jason Wei. 2022.
\newblock \href {https://doi.org/10.48550/arXiv.2210.11416} {Scaling
  instruction-finetuned language models}.
\newblock \emph{arXiv}, abs/2210.11416.

\bibitem[{Devlin et~al.(2019)Devlin, Chang, Lee, and Toutanova}]{DevlinCLT19}
Jacob Devlin, Ming{-}Wei Chang, Kenton Lee, and Kristina Toutanova. 2019.
\newblock \href {https://doi.org/10.18653/v1/n19-1423} {{BERT:} pre-training of
  deep bidirectional transformers for language understanding}.
\newblock In \emph{Proceedings of the 2019 Conference of the North American
  Chapter of the Association for Computational Linguistics: Human Language
  Technologies, {NAACL-HLT} 2019, Minneapolis, MN, USA, June 2-7, 2019, Volume
  1 (Long and Short Papers)}, pages 4171--4186.

\bibitem[{Dolan and Brockett(2005)}]{DolanB05}
William~B. Dolan and Chris Brockett. 2005.
\newblock \href {https://aclanthology.org/I05-5002/} {Automatically
  constructing a corpus of sentential paraphrases}.
\newblock In \emph{Proceedings of the Third International Workshop on
  Paraphrasing, IWP@IJCNLP 2005, Jeju Island, Korea, October 2005, 2005}.

\bibitem[{Du et~al.(2022)Du, Huang, Dai, Tong, Lepikhin, Xu, Krikun, Zhou, Yu,
  Firat, Zoph, Fedus, Bosma, Zhou, Wang, Wang, Webster, Pellat, Robinson,
  Meier{-}Hellstern, Duke, Dixon, Zhang, Le, Wu, Chen, and
  Cui}]{DuHDTLXKZYFZFBZ22}
Nan Du, Yanping Huang, Andrew~M. Dai, Simon Tong, Dmitry Lepikhin, Yuanzhong
  Xu, Maxim Krikun, Yanqi Zhou, Adams~Wei Yu, Orhan Firat, Barret Zoph, Liam
  Fedus, Maarten~P. Bosma, Zongwei Zhou, Tao Wang, Yu~Emma Wang, Kellie
  Webster, Marie Pellat, Kevin Robinson, Kathleen~S. Meier{-}Hellstern, Toju
  Duke, Lucas Dixon, Kun Zhang, Quoc~V. Le, Yonghui Wu, Zhifeng Chen, and
  Claire Cui. 2022.
\newblock \href {https://proceedings.mlr.press/v162/du22c.html} {Glam:
  Efficient scaling of language models with mixture-of-experts}.
\newblock In \emph{International Conference on Machine Learning, {ICML} 2022,
  17-23 July 2022, Baltimore, Maryland, {USA}}, volume 162 of \emph{Proceedings
  of Machine Learning Research}, pages 5547--5569.

\bibitem[{Fedus et~al.(2021)Fedus, Zoph, and Shazeer}]{Fedus21}
William Fedus, Barret Zoph, and Noam Shazeer. 2021.
\newblock \href {https://arxiv.org/abs/2101.03961} {Switch transformers:
  Scaling to trillion parameter models with simple and efficient sparsity}.
\newblock \emph{arXiv}, 2101.03961.

\bibitem[{Goyal et~al.(2020)Goyal, Choudhury, Raje, Chakaravarthy, Sabharwal,
  and Verma}]{GoyalCRCSV20}
Saurabh Goyal, Anamitra~Roy Choudhury, Saurabh Raje, Venkatesan~T.
  Chakaravarthy, Yogish Sabharwal, and Ashish Verma. 2020.
\newblock \href {http://proceedings.mlr.press/v119/goyal20a.html} {Power-bert:
  Accelerating {BERT} inference via progressive word-vector elimination}.
\newblock In \emph{Proceedings of the 37th International Conference on Machine
  Learning, {ICML} 2020, 13-18 July 2020, Virtual Event}, volume 119 of
  \emph{Proceedings of Machine Learning Research}, pages 3690--3699.

\bibitem[{Han et~al.(2015)Han, Pool, Tran, and Dally}]{HanPTD15}
Song Han, Jeff Pool, John Tran, and William~J. Dally. 2015.
\newblock \href {http://arxiv.org/abs/1506.02626} {Learning both weights and
  connections for efficient neural networks}.
\newblock \emph{arXiv}, 1506.02626.

\bibitem[{Han et~al.(2021)Han, Huang, Song, Yang, Wang, and Wang}]{Han21}
Yizeng Han, Gao Huang, Shiji Song, Le~Yang, Honghui Wang, and Yulin Wang. 2021.
\newblock \href {https://arxiv.org/abs/2102.04906} {Dynamic neural networks:
  {A} survey}.
\newblock \emph{arXiv}, 2102.04906.

\bibitem[{Hastie et~al.(2001)Hastie, Friedman, and Tibshirani}]{HastieFT01}
Trevor Hastie, Jerome~H. Friedman, and Robert Tibshirani. 2001.
\newblock \href {https://doi.org/10.1007/978-0-387-21606-5} {\emph{The Elements
  of Statistical Learning: Data Mining, Inference, and Prediction}}.
\newblock Springer Series in Statistics.

\bibitem[{He et~al.(2021)He, Qiu, Zeng, Yang, Zhai, and Tang}]{He21}
Jiaao He, Jiezhong Qiu, Aohan Zeng, Zhilin Yang, Jidong Zhai, and Jie Tang.
  2021.
\newblock \href {https://arxiv.org/abs/2103.13262} {Fastmoe: {A} fast
  mixture-of-expert training system}.
\newblock \emph{arXiv}, 2103.13262.

\bibitem[{Hinton et~al.(2015)Hinton, Vinyals, and Dean}]{HintonVD15}
Geoffrey~E. Hinton, Oriol Vinyals, and Jeffrey Dean. 2015.
\newblock \href {http://arxiv.org/abs/1503.02531} {Distilling the knowledge in
  a neural network}.
\newblock \emph{arXiv}, 1503.02531.

\bibitem[{Hou et~al.(2020)Hou, Huang, Shang, Jiang, Chen, and Liu}]{HouHSJCL20}
Lu~Hou, Zhiqi Huang, Lifeng Shang, Xin Jiang, Xiao Chen, and Qun Liu. 2020.
\newblock \href
  {https://proceedings.neurips.cc/paper/2020/hash/6f5216f8d89b086c18298e043bfe48ed-Abstract.html}
  {Dynabert: Dynamic {BERT} with adaptive width and depth}.
\newblock In \emph{Advances in Neural Information Processing Systems 33: Annual
  Conference on Neural Information Processing Systems 2020, NeurIPS 2020,
  December 6-12, 2020, virtual}.

\bibitem[{Jiao et~al.(2020)Jiao, Yin, Shang, Jiang, Chen, Li, Wang, and
  Liu}]{JiaoYSJCL0L20}
Xiaoqi Jiao, Yichun Yin, Lifeng Shang, Xin Jiang, Xiao Chen, Linlin Li, Fang
  Wang, and Qun Liu. 2020.
\newblock \href {https://doi.org/10.18653/v1/2020.findings-emnlp.372}
  {Tinybert: Distilling {BERT} for natural language understanding}.
\newblock In \emph{Findings of the Association for Computational Linguistics:
  {EMNLP} 2020, Online Event, 16-20 November 2020}, volume {EMNLP} 2020 of
  \emph{Findings of {ACL}}, pages 4163--4174.

\bibitem[{Kim and Cho(2021)}]{KimC20}
Gyuwan Kim and Kyunghyun Cho. 2021.
\newblock \href {https://doi.org/10.18653/v1/2021.acl-long.508}
  {Length-adaptive transformer: Train once with length drop, use anytime with
  search}.
\newblock In \emph{Proceedings of the 59th Annual Meeting of the Association
  for Computational Linguistics and the 11th International Joint Conference on
  Natural Language Processing, {ACL/IJCNLP} 2021, (Volume 1: Long Papers),
  Virtual Event, August 1-6, 2021}, pages 6501--6511.

\bibitem[{Lepikhin et~al.(2021)Lepikhin, Lee, Xu, Chen, Firat, Huang, Krikun,
  Shazeer, and Chen}]{LepikhinLXCFHKS21}
Dmitry Lepikhin, HyoukJoong Lee, Yuanzhong Xu, Dehao Chen, Orhan Firat, Yanping
  Huang, Maxim Krikun, Noam Shazeer, and Zhifeng Chen. 2021.
\newblock \href {https://openreview.net/forum?id=qrwe7XHTmYb} {Gshard: Scaling
  giant models with conditional computation and automatic sharding}.
\newblock In \emph{9th International Conference on Learning Representations,
  {ICLR} 2021, Virtual Event, Austria, May 3-7, 2021}.

\bibitem[{Levesque et~al.(2012)Levesque, Davis, and Morgenstern}]{LevesqueDM12}
Hector~J. Levesque, Ernest Davis, and Leora Morgenstern. 2012.
\newblock \href {http://www.aaai.org/ocs/index.php/KR/KR12/paper/view/4492}
  {The winograd schema challenge}.
\newblock In \emph{Principles of Knowledge Representation and Reasoning:
  Proceedings of the Thirteenth International Conference, {KR} 2012, Rome,
  Italy, June 10-14, 2012}.

\bibitem[{Li et~al.(2020)Li, Liu, Zhao, Xu, Yang, and Jin}]{LiLZXYJ20}
Jianquan Li, Xiaokang Liu, Honghong Zhao, Ruifeng Xu, Min Yang, and Yaohong
  Jin. 2020.
\newblock \href {https://doi.org/10.18653/v1/2020.emnlp-main.242} {{BERT-EMD:}
  many-to-many layer mapping for {BERT} compression with earth mover's
  distance}.
\newblock In \emph{Proceedings of the 2020 Conference on Empirical Methods in
  Natural Language Processing, {EMNLP} 2020, Online, November 16-20, 2020},
  pages 3009--3018.

\bibitem[{Liu et~al.(2019)Liu, Ott, Goyal, Du, Joshi, Chen, Levy, Lewis,
  Zettlemoyer, and Stoyanov}]{Liu19}
Yinhan Liu, Myle Ott, Naman Goyal, Jingfei Du, Mandar Joshi, Danqi Chen, Omer
  Levy, Mike Lewis, Luke Zettlemoyer, and Veselin Stoyanov. 2019.
\newblock \href {http://arxiv.org/abs/1907.11692} {Roberta: {A} robustly
  optimized {BERT} pretraining approach}.
\newblock \emph{arXiv}, 1907.11692.

\bibitem[{Lopez{-}Paz et~al.(2016)Lopez{-}Paz, Bottou, Sch{\"{o}}lkopf, and
  Vapnik}]{Lopez-PazBSV15}
David Lopez{-}Paz, L{\'{e}}on Bottou, Bernhard Sch{\"{o}}lkopf, and Vladimir
  Vapnik. 2016.
\newblock \href {http://arxiv.org/abs/1511.03643} {Unifying distillation and
  privileged information}.
\newblock In \emph{4th International Conference on Learning Representations,
  {ICLR} 2016, San Juan, Puerto Rico, May 2-4, 2016, Conference Track
  Proceedings}.

\bibitem[{Mirzadeh et~al.(2020)Mirzadeh, Farajtabar, Li, Levine, Matsukawa, and
  Ghasemzadeh}]{MirzadehFLLMG20}
Seyed{-}Iman Mirzadeh, Mehrdad Farajtabar, Ang Li, Nir Levine, Akihiro
  Matsukawa, and Hassan Ghasemzadeh. 2020.
\newblock \href {https://ojs.aaai.org/index.php/AAAI/article/view/5963}
  {Improved knowledge distillation via teacher assistant}.
\newblock In \emph{The Thirty-Fourth {AAAI} Conference on Artificial
  Intelligence, {AAAI} 2020, The Thirty-Second Innovative Applications of
  Artificial Intelligence Conference, {IAAI} 2020, The Tenth {AAAI} Symposium
  on Educational Advances in Artificial Intelligence, {EAAI} 2020, New York,
  NY, USA, February 7-12, 2020}, pages 5191--5198.

\bibitem[{Park et~al.(2021{\natexlab{a}})Park, Cha, Jeong, Kim, and
  Han}]{ParkCJKH21}
Dae~Young Park, Moon{-}Hyun Cha, Changwook Jeong, Daesin Kim, and Bohyung Han.
  2021{\natexlab{a}}.
\newblock \href
  {https://proceedings.neurips.cc/paper/2021/hash/6e7d2da6d3953058db75714ac400b584-Abstract.html}
  {Learning student-friendly teacher networks for knowledge distillation}.
\newblock In \emph{Advances in Neural Information Processing Systems 34: Annual
  Conference on Neural Information Processing Systems 2021, NeurIPS 2021,
  December 6-14, 2021, virtual}, pages 13292--13303.

\bibitem[{Park et~al.(2021{\natexlab{b}})Park, Kim, and Yang}]{ParkKY21}
Geondo Park, Gyeongman Kim, and Eunho Yang. 2021{\natexlab{b}}.
\newblock \href {https://doi.org/10.18653/v1/2021.emnlp-main.30} {Distilling
  linguistic context for language model compression}.
\newblock In \emph{Proceedings of the 2021 Conference on Empirical Methods in
  Natural Language Processing, {EMNLP} 2021, Virtual Event / Punta Cana,
  Dominican Republic, 7-11 November, 2021}, pages 364--378.

\bibitem[{Raffel et~al.(2020)Raffel, Shazeer, Roberts, Lee, Narang, Matena,
  Zhou, Li, and Liu}]{RaffelSRLNMZLL20}
Colin Raffel, Noam Shazeer, Adam Roberts, Katherine Lee, Sharan Narang, Michael
  Matena, Yanqi Zhou, Wei Li, and Peter~J. Liu. 2020.
\newblock \href {http://jmlr.org/papers/v21/20-074.html} {Exploring the limits
  of transfer learning with a unified text-to-text transformer}.
\newblock \emph{Journal of Machine Learning Research}, 21:140:1--140:67.

\bibitem[{Rajbhandari et~al.(2022)Rajbhandari, Li, Yao, Zhang, Aminabadi, Awan,
  Rasley, and He}]{RajbhandariLYZA22}
Samyam Rajbhandari, Conglong Li, Zhewei Yao, Minjia Zhang, Reza~Yazdani
  Aminabadi, Ammar~Ahmad Awan, Jeff Rasley, and Yuxiong He. 2022.
\newblock \href {https://proceedings.mlr.press/v162/rajbhandari22a.html}
  {Deepspeed-moe: Advancing mixture-of-experts inference and training to power
  next-generation {AI} scale}.
\newblock In \emph{International Conference on Machine Learning, {ICML} 2022,
  17-23 July 2022, Baltimore, Maryland, {USA}}, volume 162 of \emph{Proceedings
  of Machine Learning Research}, pages 18332--18346.

\bibitem[{Rajpurkar et~al.(2016)Rajpurkar, Zhang, Lopyrev, and
  Liang}]{RajpurkarZLL16}
Pranav Rajpurkar, Jian Zhang, Konstantin Lopyrev, and Percy Liang. 2016.
\newblock \href {https://doi.org/10.18653/v1/d16-1264} {Squad: 100, 000+
  questions for machine comprehension of text}.
\newblock In \emph{Proceedings of the 2016 Conference on Empirical Methods in
  Natural Language Processing, {EMNLP} 2016, Austin, Texas, USA, November 1-4,
  2016}, pages 2383--2392.

\bibitem[{Roller et~al.(2021)Roller, Sukhbaatar, Szlam, and
  Weston}]{RollerSSW21}
Stephen Roller, Sainbayar Sukhbaatar, Arthur Szlam, and Jason Weston. 2021.
\newblock \href
  {https://proceedings.neurips.cc/paper/2021/hash/92bf5e6240737e0326ea59846a83e076-Abstract.html}
  {Hash layers for large sparse models}.
\newblock In \emph{Advances in Neural Information Processing Systems 34: Annual
  Conference on Neural Information Processing Systems 2021, NeurIPS 2021,
  December 6-14, 2021, virtual}, pages 17555--17566.

\bibitem[{Sang and Meulder(2003)}]{SangM03}
Erik F. Tjong~Kim Sang and Fien~De Meulder. 2003.
\newblock \href {https://aclanthology.org/W03-0419/} {Introduction to the
  conll-2003 shared task: Language-independent named entity recognition}.
\newblock In \emph{Proceedings of the Seventh Conference on Natural Language
  Learning, CoNLL 2003, Held in cooperation with {HLT-NAACL} 2003, Edmonton,
  Canada, May 31 - June 1, 2003}, pages 142--147.

\bibitem[{Sanh et~al.(2019)Sanh, Debut, Chaumond, and Wolf}]{Sanh19}
Victor Sanh, Lysandre Debut, Julien Chaumond, and Thomas Wolf. 2019.
\newblock \href {http://arxiv.org/abs/1910.01108} {Distilbert, a distilled
  version of {BERT:} smaller, faster, cheaper and lighter}.
\newblock \emph{arXiv}, 1910.01108.

\bibitem[{Shazeer et~al.(2018)Shazeer, Cheng, Parmar, Tran, Vaswani,
  Koanantakool, Hawkins, Lee, Hong, Young, Sepassi, and
  Hechtman}]{ShazeerCPTVKHLH18}
Noam Shazeer, Youlong Cheng, Niki Parmar, Dustin Tran, Ashish Vaswani, Penporn
  Koanantakool, Peter Hawkins, HyoukJoong Lee, Mingsheng Hong, Cliff Young,
  Ryan Sepassi, and Blake~A. Hechtman. 2018.
\newblock \href
  {https://proceedings.neurips.cc/paper/2018/hash/3a37abdeefe1dab1b30f7c5c7e581b93-Abstract.html}
  {Mesh-tensorflow: Deep learning for supercomputers}.
\newblock In \emph{Advances in Neural Information Processing Systems 31: Annual
  Conference on Neural Information Processing Systems 2018, NeurIPS 2018,
  December 3-8, 2018, Montr{\'{e}}al, Canada}, pages 10435--10444.

\bibitem[{Shazeer et~al.(2017)Shazeer, Mirhoseini, Maziarz, Davis, Le, Hinton,
  and Dean}]{ShazeerMMDLHD17}
Noam Shazeer, Azalia Mirhoseini, Krzysztof Maziarz, Andy Davis, Quoc~V. Le,
  Geoffrey~E. Hinton, and Jeff Dean. 2017.
\newblock \href {https://openreview.net/forum?id=B1ckMDqlg} {Outrageously large
  neural networks: The sparsely-gated mixture-of-experts layer}.
\newblock In \emph{5th International Conference on Learning Representations,
  {ICLR} 2017, Toulon, France, April 24-26, 2017, Conference Track
  Proceedings}.

\bibitem[{Shoeybi et~al.(2019)Shoeybi, Patwary, Puri, LeGresley, Casper, and
  Catanzaro}]{Shoeybi19}
Mohammad Shoeybi, Mostofa Patwary, Raul Puri, Patrick LeGresley, Jared Casper,
  and Bryan Catanzaro. 2019.
\newblock \href {http://arxiv.org/abs/1909.08053} {Megatron-lm: Training
  multi-billion parameter language models using model parallelism}.
\newblock \emph{arXiv}, abs/1909.08053.

\bibitem[{Socher et~al.(2013)Socher, Perelygin, Wu, Chuang, Manning, Ng, and
  Potts}]{SocherPWCMNP13}
Richard Socher, Alex Perelygin, Jean Wu, Jason Chuang, Christopher~D. Manning,
  Andrew~Y. Ng, and Christopher Potts. 2013.
\newblock \href {https://aclanthology.org/D13-1170/} {Recursive deep models for
  semantic compositionality over a sentiment treebank}.
\newblock In \emph{Proceedings of the 2013 Conference on Empirical Methods in
  Natural Language Processing, {EMNLP} 2013, 18-21 October 2013, Grand Hyatt
  Seattle, Seattle, Washington, USA, {A} meeting of SIGDAT, a Special Interest
  Group of the {ACL}}, pages 1631--1642.

\bibitem[{Sun et~al.(2019)Sun, Cheng, Gan, and Liu}]{SunCGL19}
Siqi Sun, Yu~Cheng, Zhe Gan, and Jingjing Liu. 2019.
\newblock \href {https://doi.org/10.18653/v1/D19-1441} {Patient knowledge
  distillation for {BERT} model compression}.
\newblock In \emph{Proceedings of the 2019 Conference on Empirical Methods in
  Natural Language Processing and the 9th International Joint Conference on
  Natural Language Processing, {EMNLP-IJCNLP} 2019, Hong Kong, China, November
  3-7, 2019}, pages 4322--4331.

\bibitem[{Sun et~al.(2020{\natexlab{a}})Sun, Gan, Fang, Cheng, Wang, and
  Liu}]{SunGFCWL20}
Siqi Sun, Zhe Gan, Yuwei Fang, Yu~Cheng, Shuohang Wang, and Jingjing Liu.
  2020{\natexlab{a}}.
\newblock \href {https://doi.org/10.18653/v1/2020.emnlp-main.36} {Contrastive
  distillation on intermediate representations for language model compression}.
\newblock In \emph{Proceedings of the 2020 Conference on Empirical Methods in
  Natural Language Processing, {EMNLP} 2020, Online, November 16-20, 2020},
  pages 498--508.

\bibitem[{Sun et~al.(2020{\natexlab{b}})Sun, Yu, Song, Liu, Yang, and
  Zhou}]{SunYSLYZ20}
Zhiqing Sun, Hongkun Yu, Xiaodan Song, Renjie Liu, Yiming Yang, and Denny Zhou.
  2020{\natexlab{b}}.
\newblock \href {https://doi.org/10.18653/v1/2020.acl-main.195} {Mobilebert: a
  compact task-agnostic {BERT} for resource-limited devices}.
\newblock In \emph{Proceedings of the 58th Annual Meeting of the Association
  for Computational Linguistics, {ACL} 2020, Online, July 5-10, 2020}, pages
  2158--2170.

\bibitem[{Sung et~al.(2015)Sung, Shin, and Hwang}]{SungSH15}
Wonyong Sung, Sungho Shin, and Kyuyeon Hwang. 2015.
\newblock \href {http://arxiv.org/abs/1511.06488} {Resiliency of deep neural
  networks under quantization}.
\newblock \emph{arXiv}, 1511.06488.

\bibitem[{Turc et~al.(2019)Turc, Chang, Lee, and Toutanova}]{Turc19}
Iulia Turc, Ming{-}Wei Chang, Kenton Lee, and Kristina Toutanova. 2019.
\newblock \href {http://arxiv.org/abs/1908.08962} {Well-read students learn
  better: The impact of student initialization on knowledge distillation}.
\newblock \emph{arXiv}, 1908.08962.

\bibitem[{Vapnik(1998)}]{Vapnik98}
Vladimir Vapnik. 1998.
\newblock \emph{Statistical learning theory}.
\newblock Wiley.

\bibitem[{Vaswani et~al.(2017)Vaswani, Shazeer, Parmar, Uszkoreit, Jones,
  Gomez, Kaiser, and Polosukhin}]{VaswaniSPUJGKP17}
Ashish Vaswani, Noam Shazeer, Niki Parmar, Jakob Uszkoreit, Llion Jones,
  Aidan~N. Gomez, Lukasz Kaiser, and Illia Polosukhin. 2017.
\newblock \href
  {https://proceedings.neurips.cc/paper/2017/hash/3f5ee243547dee91fbd053c1c4a845aa-Abstract.html}
  {Attention is all you need}.
\newblock In \emph{Advances in Neural Information Processing Systems 30: Annual
  Conference on Neural Information Processing Systems 2017, December 4-9, 2017,
  Long Beach, CA, {USA}}, pages 5998--6008.

\bibitem[{Wang et~al.(2019)Wang, Singh, Michael, Hill, Levy, and
  Bowman}]{WangSMHLB19}
Alex Wang, Amanpreet Singh, Julian Michael, Felix Hill, Omer Levy, and
  Samuel~R. Bowman. 2019.
\newblock \href {https://openreview.net/forum?id=rJ4km2R5t7} {{GLUE:} {A}
  multi-task benchmark and analysis platform for natural language
  understanding}.
\newblock In \emph{7th International Conference on Learning Representations,
  {ICLR} 2019, New Orleans, LA, USA, May 6-9, 2019}.

\bibitem[{Wang et~al.(2021{\natexlab{a}})Wang, Sun, Xiang, Wu, Ding, Gong,
  Feng, Shang, Zhao, Pang, Liu, Chen, Lu, Liu, Wang, Bai, Chen, Zhao, Li, Sun,
  Yu, Ma, Tian, Wu, Wu, Zeng, Li, Gao, and Wang}]{Wang21}
Shuohuan Wang, Yu~Sun, Yang Xiang, Zhihua Wu, Siyu Ding, Weibao Gong, Shikun
  Feng, Junyuan Shang, Yanbin Zhao, Chao Pang, Jiaxiang Liu, Xuyi Chen, Yuxiang
  Lu, Weixin Liu, Xi~Wang, Yangfan Bai, Qiuliang Chen, Li~Zhao, Shiyong Li,
  Peng Sun, Dianhai Yu, Yanjun Ma, Hao Tian, Hua Wu, Tian Wu, Wei Zeng, Ge~Li,
  Wen Gao, and Haifeng Wang. 2021{\natexlab{a}}.
\newblock \href {https://arxiv.org/abs/2112.12731} {{ERNIE} 3.0 titan:
  Exploring larger-scale knowledge enhanced pre-training for language
  understanding and generation}.
\newblock \emph{arXiv}, 2112.12731.

\bibitem[{Wang et~al.(2021{\natexlab{b}})Wang, Bao, Huang, Dong, and
  Wei}]{WangBHDW21}
Wenhui Wang, Hangbo Bao, Shaohan Huang, Li~Dong, and Furu Wei.
  2021{\natexlab{b}}.
\newblock \href {https://doi.org/10.18653/v1/2021.findings-acl.188} {Minilmv2:
  Multi-head self-attention relation distillation for compressing pretrained
  transformers}.
\newblock In \emph{Findings of the Association for Computational Linguistics:
  {ACL/IJCNLP} 2021, Online Event, August 1-6, 2021}, volume {ACL/IJCNLP} 2021
  of \emph{Findings of {ACL}}, pages 2140--2151.

\bibitem[{Wang et~al.(2020)Wang, Wei, Dong, Bao, Yang, and Zhou}]{WangW0B0020}
Wenhui Wang, Furu Wei, Li~Dong, Hangbo Bao, Nan Yang, and Ming Zhou. 2020.
\newblock \href
  {https://proceedings.neurips.cc/paper/2020/hash/3f5ee243547dee91fbd053c1c4a845aa-Abstract.html}
  {Minilm: Deep self-attention distillation for task-agnostic compression of
  pre-trained transformers}.
\newblock In \emph{Advances in Neural Information Processing Systems 33: Annual
  Conference on Neural Information Processing Systems 2020, NeurIPS 2020,
  December 6-12, 2020, virtual}.

\bibitem[{Warstadt et~al.(2019)Warstadt, Singh, and Bowman}]{WarstadtSB19}
Alex Warstadt, Amanpreet Singh, and Samuel~R. Bowman. 2019.
\newblock \href {https://doi.org/10.1162/tacl\_a\_00290} {Neural network
  acceptability judgments}.
\newblock \emph{Transactions on Association for Computational Linguistics},
  7:625--641.

\bibitem[{Williams et~al.(2018)Williams, Nangia, and Bowman}]{WilliamsNB18}
Adina Williams, Nikita Nangia, and Samuel~R. Bowman. 2018.
\newblock \href {https://doi.org/10.18653/v1/n18-1101} {A broad-coverage
  challenge corpus for sentence understanding through inference}.
\newblock In \emph{Proceedings of the 2018 Conference of the North American
  Chapter of the Association for Computational Linguistics: Human Language
  Technologies, {NAACL-HLT} 2018, New Orleans, Louisiana, USA, June 1-6, 2018,
  Volume 1 (Long Papers)}, pages 1112--1122.

\bibitem[{Xia et~al.(2022)Xia, Zhong, and Chen}]{XiaZC22}
Mengzhou Xia, Zexuan Zhong, and Danqi Chen. 2022.
\newblock \href {https://doi.org/10.18653/v1/2022.acl-long.107} {Structured
  pruning learns compact and accurate models}.
\newblock In \emph{Proceedings of the 60th Annual Meeting of the Association
  for Computational Linguistics (Volume 1: Long Papers), {ACL} 2022, Dublin,
  Ireland, May 22-27, 2022}, pages 1513--1528.

\bibitem[{Xin et~al.(2020)Xin, Tang, Lee, Yu, and Lin}]{XinTLYL20}
Ji~Xin, Raphael Tang, Jaejun Lee, Yaoliang Yu, and Jimmy Lin. 2020.
\newblock \href {https://doi.org/10.18653/v1/2020.acl-main.204} {Deebert:
  Dynamic early exiting for accelerating {BERT} inference}.
\newblock In \emph{Proceedings of the 58th Annual Meeting of the Association
  for Computational Linguistics, {ACL} 2020, Online, July 5-10, 2020}, pages
  2246--2251.

\bibitem[{Xu et~al.(2020)Xu, Hu, Zhang, Li, Cao, Li, Xu, Sun, Yu, Yu, Tian,
  Dong, Liu, Shi, Cui, Li, Zeng, Wang, Xie, Li, Patterson, Tian, Zhang, Zhou,
  Liu, Zhao, Zhao, Yue, Zhang, Yang, Richardson, and Lan}]{XuHZLCLXSYYTDLS20}
Liang Xu, Hai Hu, Xuanwei Zhang, Lu~Li, Chenjie Cao, Yudong Li, Yechen Xu, Kai
  Sun, Dian Yu, Cong Yu, Yin Tian, Qianqian Dong, Weitang Liu, Bo~Shi, Yiming
  Cui, Junyi Li, Jun Zeng, Rongzhao Wang, Weijian Xie, Yanting Li, Yina
  Patterson, Zuoyu Tian, Yiwen Zhang, He~Zhou, Shaoweihua Liu, Zhe Zhao, Qipeng
  Zhao, Cong Yue, Xinrui Zhang, Zhengliang Yang, Kyle Richardson, and Zhenzhong
  Lan. 2020.
\newblock \href {https://doi.org/10.18653/v1/2020.coling-main.419} {{CLUE:} {A}
  chinese language understanding evaluation benchmark}.
\newblock In \emph{Proceedings of the 28th International Conference on
  Computational Linguistics, {COLING} 2020, Barcelona, Spain (Online), December
  8-13, 2020}, pages 4762--4772. International Committee on Computational
  Linguistics.

\bibitem[{Xue et~al.(2022)Xue, He, Ren, Lou, and You}]{Xue22}
Fuzhao Xue, Xiaoxin He, Xiaozhe Ren, Yuxuan Lou, and Yang You. 2022.
\newblock \href {https://arxiv.org/abs/2201.10890} {One student knows all
  experts know: From sparse to dense}.
\newblock \emph{arXiv}, 2201.10890.

\bibitem[{Yang et~al.(2022)Yang, Zhang, and Song}]{Yang22}
Yi~Yang, Chen Zhang, and Dawei Song. 2022.
\newblock \href {https://doi.org/10.48550/arXiv.2210.03923} {Sparse teachers
  can be dense with knowledge}.
\newblock \emph{arXiv}, abs/2210.03923.

\bibitem[{Yuan et~al.(2021)Yuan, Zhao, Du, Ding, Liu, Cen, Zou, Yang, and
  Tang}]{YuanZDDLCZYT21}
Sha Yuan, Hanyu Zhao, Zhengxiao Du, Ming Ding, Xiao Liu, Yukuo Cen, Xu~Zou,
  Zhilin Yang, and Jie Tang. 2021.
\newblock \href {https://doi.org/10.1016/j.aiopen.2021.06.001} {Wudaocorpora:
  {A} super large-scale chinese corpora for pre-training language models}.
\newblock \emph{{AI} Open}, 2:65--68.

\bibitem[{Zafrir et~al.(2019)Zafrir, Boudoukh, Izsak, and
  Wasserblat}]{ZafrirBIW19}
Ofir Zafrir, Guy Boudoukh, Peter Izsak, and Moshe Wasserblat. 2019.
\newblock \href {https://doi.org/10.1109/EMC2-NIPS53020.2019.00016} {{Q8BERT:}
  quantized 8bit {BERT}}.
\newblock In \emph{Fifth Workshop on Energy Efficient Machine Learning and
  Cognitive Computing - NeurIPS Edition, EMC2@NeurIPS 2019, Vancouver, Canada,
  December 13, 2019}, pages 36--39.

\bibitem[{Zhang et~al.(2022{\natexlab{a}})Zhang, Yang, Wang, Liu, Wang, Xian,
  Wu, and Song}]{Zhang22}
Chen Zhang, Yang Yang, Qifan Wang, Jiahao Liu, Jingang Wang, Yunsen Xian, Wei
  Wu, and Dawei Song. 2022{\natexlab{a}}.
\newblock \href {https://doi.org/10.48550/arXiv.2205.14570} {Minidisc: Minimal
  distillation schedule for language model compression}.
\newblock \emph{arXiv}, 2205.14570.

\bibitem[{Zhang et~al.(2022{\natexlab{b}})Zhang, Lin, Liu, Li, Sun, and
  Zhou}]{ZhangL00S022}
Zhengyan Zhang, Yankai Lin, Zhiyuan Liu, Peng Li, Maosong Sun, and Jie Zhou.
  2022{\natexlab{b}}.
\newblock \href {https://doi.org/10.18653/v1/2022.findings-acl.71}
  {Moefication: Transformer feed-forward layers are mixtures of experts}.
\newblock In \emph{Findings of the Association for Computational Linguistics:
  {ACL} 2022, Dublin, Ireland, May 22-27, 2022}, pages 877--890.

\bibitem[{Zhao et~al.(2022)Zhao, Cui, Song, Qiu, and Liang}]{Zhao22}
Borui Zhao, Quan Cui, Renjie Song, Yiyu Qiu, and Jiajun Liang. 2022.
\newblock \href {https://doi.org/10.48550/arXiv.2203.08679} {Decoupled
  knowledge distillation}.
\newblock \emph{arXiv}, 2203.08679.

\bibitem[{Zhou et~al.(2020)Zhou, Xu, Ge, McAuley, Xu, and Wei}]{ZhouXGM0W20}
Wangchunshu Zhou, Canwen Xu, Tao Ge, Julian~J. McAuley, Ke~Xu, and Furu Wei.
  2020.
\newblock \href
  {https://proceedings.neurips.cc/paper/2020/hash/d4dd111a4fd973394238aca5c05bebe3-Abstract.html}
  {{BERT} loses patience: Fast and robust inference with early exit}.
\newblock In \emph{Advances in Neural Information Processing Systems 33: Annual
  Conference on Neural Information Processing Systems 2020, NeurIPS 2020,
  December 6-12, 2020, virtual}.

\bibitem[{Zhou et~al.(2022)Zhou, Xu, and McAuley}]{ZhouXM22}
Wangchunshu Zhou, Canwen Xu, and Julian~J. McAuley. 2022.
\newblock \href {https://doi.org/10.18653/v1/2022.acl-long.485} {{BERT} learns
  to teach: Knowledge distillation with meta learning}.
\newblock In \emph{Proceedings of the 60th Annual Meeting of the Association
  for Computational Linguistics (Volume 1: Long Papers), {ACL} 2022, Dublin,
  Ireland, May 22-27, 2022}, pages 7037--7049.

\bibitem[{Zhu and Wang(2021)}]{ZhuW21a}
Yichen Zhu and Yi~Wang. 2021.
\newblock \href {https://doi.org/10.1109/ICCV48922.2021.00501} {Student
  customized knowledge distillation: Bridging the gap between student and
  teacher}.
\newblock In \emph{2021 {IEEE/CVF} International Conference on Computer Vision,
  {ICCV} 2021, Montreal, QC, Canada, October 10-17, 2021}, pages 5037--5046.

\bibitem[{Zuo et~al.(2022)Zuo, Zhang, Liang, He, Zhao, and Chen}]{ZuoZLHZC22}
Simiao Zuo, Qingru Zhang, Chen Liang, Pengcheng He, Tuo Zhao, and Weizhu Chen.
  2022.
\newblock \href {https://doi.org/10.18653/v1/2022.naacl-main.116} {Moebert:
  from {BERT} to mixture-of-experts via importance-guided adaptation}.
\newblock In \emph{Proceedings of the 2022 Conference of the North American
  Chapter of the Association for Computational Linguistics: Human Language
  Technologies, {NAACL} 2022, Seattle, WA, United States, July 10-15, 2022},
  pages 1610--1623.

\end{thebibliography}
\bibliographystyle{acl_natbib}

\appendix

\section{Data Summary}
\label{app_a}

The detailed statistics, maximum sequence lengths, and metrics for datasets we use are shown in Table~\ref{tab_a}, where the Wikipedia corpus used for distillation is also attached.

\begin{table*}[ht]
    \centering
    \caption{The statistics, maximum sequence lengths, and metrics.}
    \begin{adjustbox}{width=0.8\textwidth,center}
    \begin{tabular}{lrrcc}
      \toprule
        \textbf{Dataset} & \textbf{\#Train exam.} & \textbf{\#Dev exam.} & \textbf{Max. length} & \textbf{Metric} \\
      \midrule
        SST-2 &  67K & 0.9K & 64 & Accuracy \\
        MRPC & 3.7K & 0.4K & 128 & F1 \\
        STS-B & 7K & 1.5K & 128 & Spearman Correlation \\
        QQP & 364K & 40K & 128 & F1 \\
        MNLI-m/mm & 393K & 20K & 128 & Accuracy \\
        QNLI & 105K & 5.5K & 128 & Accuracy \\
        RTE & 2.5K & 0.3K & 128 & Accuracy \\
        CoNLL & 14k & 3.3k & 128 & F1 \\
     \midrule
        Wikipedia & 35M & - & 128 & - \\
      \bottomrule
    \end{tabular}
    \end{adjustbox}
    \label{tab_a}
\end{table*}

\section{More Hands-on Details}
\label{app_b}

\paragraph{General Guidelines} The details of hyperparameters for distillation and finetuning are shown in Table~\ref{tab_b}. We will be releasing our code and scripts in the final version for exact reproducibility. For all cases, students are always randomly initialized following MiniLM.

\begin{table*}[ht]
    \centering
    \caption{The hyperparameters for both distillation and finetuning. The search grids for GLUE and CoNLL are indicated differently.}
    \begin{adjustbox}{width=0.7\textwidth,center}
    \begin{tabular}{lcc}
      \toprule
        \textbf{Hyperparameter} & \textbf{Distillation} & \textbf{Finetuning} \\
      \midrule
        Batch size & 8$\times$128$=$1024 & \{16,32\} \\
        Optimizer & AdamW & AdamW \\
        Learning rate & 3e-4 & \{1e-5,2e-5,3e-5\}/\{1e-4,2e-4,3e-4\} \\
        Training epochs & 5 & 10 \\
        Earlystop epochs & - & 5 \\
        Warmup proportion & 0.01 & 0.1 \\
        Weight decay & 0.01 & 0.01 \\
      \bottomrule
    \end{tabular}
    \end{adjustbox}
    \label{tab_b}
\end{table*}

\paragraph{Implementation of MiniMoE}

We strictly follow the design of SwitchTransformer~\citep{Fedus21} and extend it to the design of our \textsc{MiniMoE}. We also follow their associated appendices to implement an MoE for multihead attention. In detail, based on the original design, we treat an FFN/MHA as an minimal expert, adopt top-\textit{one} gating with load balancing, and employ a capacity factor of 1.25 for a good tradeoff (where overflowed tokens are dropped). For the parameter effect of adding an expert, we take expanding MiniLM\textsubscript{\sf 4L;192H} (11.3M) to MiniMoE\textsubscript{\sf 4L;192H}-{\sf 1,2E} (14.9M) as an example. The number of parameters for embeddings is not changed (6.0M$\to$6.0M), but adding an expert ({\sf 1,1E$\to$1,2E}) results in an increased number of parameters for transformers (5.4M$\to$9.0M).

Further, our design for HashLayer~\citep{RollerSSW21} also strictly follows the original random hash design, i.e., per-token hash is used. We strictly follow the best configuration of DeKD as reported in their paper~\citep{Zhao22}, where $\alpha$ is 1.0 and $\beta$ is 8.0.

\section{Results w/ Task-specific Distillation}
\label{app_c}

The results with task-specific distillation are produced from released checkpoints. The results in Table~\ref{tab_c} demonstrate that TinyBERT is largely supported with data augmentation in the task-specific distillation stage for great performance. Another intriguing observation is that data augmentation only works for distillation but not for finetuning potentially due to the noise-resilience of distillation, so we preferably replace the finetuning stage with a task-specific distillation stage in experimenting with MiniLM.

\begin{table*}[ht]
    \caption{The results with and without task-specific distillation upon distilling BERT\textsubscript{\sf base}.}
    \begin{adjustbox}{width=0.95\textwidth,center}
    \begin{threeparttable}
    \begin{tabular}{ll|ccccccc|c}
    \toprule
      \textbf{Method} & \textbf{GFLOPs} & \makecell[c]{\textbf{SST-2}\\\textbf{Acc}} & \makecell[c]{\textbf{MRPC}\\\textbf{F1}} & \makecell[c]{\textbf{STS-B}\\\textbf{SpCorr}} & \makecell[c]{\textbf{QQP}\\\textbf{F1}} & \makecell[c]{\textbf{MNLI-m/mm}\\\textbf{Acc}} & \makecell[c]{\textbf{QNLI}\\\textbf{Acc}} & \makecell[c]{\textbf{RTE}\\\textbf{Acc}} & \makecell[c]{\textbf{GLUE}\\\textbf{Score}} \\
    \midrule
      TinyBERT\textsubscript{\sf 4L;312H} & 0.60 & 88.5 & 87.9 & 86.6 & 85.6 & 78.9/79.2 & 87.3 & 67.2 & 82.7 \\
      \quad w/ tsd.+aug. & 0.60 & 91.6 & 90.2 & 86.3 & 87.1 & 81.2/82.8 & 87.6 & 64.3 & 83.9 \\
      MiniDisc\textsubscript{\sf 5\%} & 0.54 & 86.9 & 87.6 & 84.8 & 83.5 & 72.7/74.5 & 84.0 & 66.8 & 80.1 \\
      \quad w/ aug. & 0.54 & 91.2 & 90.0 & 87.5 & 85.4 & 79.0/79.8 & 84.5 & 67.5 & 83.1 \\
      MiniLM\textsubscript{\sf 3L;384H} & 0.68 & 89.1 & 89.1 & 86.6 & 85.4 & 77.8/78.4 & 87.2 & 66.1 & 82.5 \\
      \quad w/ aug. & 0.68 & 88.7 & 85.9 & 83.1 & 82.8 & 76.2/76.0 & 86.6 & 62.5 & 80.2 \\
      \quad w/ tsd.+aug.\tnote{*} & 0.68 & 91.2 & 91.1 & 88.2 & 86.6 & 79.9/80.4 & 87.8 & 66.1 & 83.9 \\
    \bottomrule
    \end{tabular}
    \begin{tablenotes}
      \item [*] tsd. indicates task-specific distillation and aug. indicates distillation with data augmentation.
    \end{tablenotes}
    \end{threeparttable}
    \end{adjustbox}
    \label{tab_c}
\end{table*}

\section{\textsc{MiniMoE} at Extreme}
\label{app_d}

The results in Table~\ref{app_d} witness that, \textsc{MiniMoE} sometimes struggles with extreme cases but can be enhanced with the help of TA.

\begin{table*}[ht]
    \caption{The results of \textsc{MiniMoE} at extreme upon distilling BERT\textsubscript{\sf base} and BERT\textsubscript{\sf large} respectively.}
    \begin{adjustbox}{width=0.95\textwidth,center}
    \begin{tabular}{lll|ccccccc|c}
    \toprule
      \textbf{Method} & \multicolumn{2}{l|}{\textbf{GFLOPs}} & \makecell[c]{\textbf{SST-2}\\\textbf{Acc}} & \makecell[c]{\textbf{MRPC}\\\textbf{F1}} & \makecell[c]{\textbf{STS-B}\\\textbf{SpCorr}} & \makecell[c]{\textbf{QQP}\\\textbf{F1}} & \makecell[c]{\textbf{MNLI-m/mm}\\\textbf{Acc}} & \makecell[c]{\textbf{QNLI}\\\textbf{Acc}} & \makecell[c]{\textbf{RTE}\\\textbf{Acc}} & \makecell[c]{\textbf{GLUE}\\\textbf{Score}} \\
    \midrule
    BERT\textsubscript{\sf base} & 10.9 & & 93.8 & 91.5 & 87.1 & 88.4 & 84.9/84.9 & 91.9 & 71.5 & 86.7 \\
    \midrule
      MiniLM\textsubscript{\sf 4L;96H} & 0.06 & & 83.4 & 84.6 & 81.9 & 80.7 & 71.2/72.5 & 82.0 & 63.7 & 77.5 \\
      \quad w/ TA & 0.06 & & 84.5 & 83.9 & 82.2 & 80.5 & 70.8/72.4 & 81.6 & 63.7 & 77.5 \\
      \textsc{MiniMoE}\textsubscript{\sf 4L;96H} & 0.06 & & 84.8 & 84.0 & 83.1 & 81.2 & 72.2/73.5 & 82.2 & 65.7 & 78.3 \\
      \quad w/ TA & 0.06 & \multirow{-4}{*}{\rotatebox[origin=c]{90}{$\sim$182$\times$}} & 84.2 & 85.3 & 83.7 & 82.2 & 72.6/73.7 & 83.6 & 65.3 & 78.8 \\
    \midrule
      MiniLM\textsubscript{\sf 3L;96H} & 0.04 & & 83.7 & 83.8 & 81.2 & 80.6 & 70.3/71.5 & 80.5 & 61.4 & 76.6 \\
      \quad w/ TA & 0.04 & & 82.6 & 83.3 & 81.2 & 80.3 & 70.3/71.9 & 80.7 & 61.4 & 76.5 \\
      \textsc{MiniMoE}\textsubscript{\sf 3L;96H} & 0.04 & & 84.8 & 84.5 & 82.8 & 80.8 & 70.3/71.9 & 81.9 & 65.0 & 77.7 \\
      \quad w/ TA & 0.04 & \multirow{-4}{*}{\rotatebox[origin=c]{90}{$\sim$273$\times$}} & 83.5 & 85.1 & 83.1 & 81.4 & 71.4/73.0 & 83.3 & 61.7 & 77.8 \\
    \midrule
    \midrule
    BERT\textsubscript{\sf large} & 38.7 & & 94.2 & 92.5 & 90.1 & 89.0 & 86.6/86.3 & 92.5 & 75.5 & 88.3 \\
    \midrule
      MiniLM\textsubscript{\sf 4L;96H} & 0.06 & & 83.3 & 83.9 & 82.5 & 81.0 & 71.4/72.4 & 81.8 & 63.2 & 77.4 \\
      \quad w/ TA & 0.06 & & 84.1 & 85.8 & 82.4 & 81.3 & 71.9/73.4 & 82.3 & 64.3 & 78.2 \\
      \textsc{MiniMoE}\textsubscript{\sf 4L;96H} & 0.06 & & 84.9 & 85.4 & 82.9 & 81.6 & 74.0/74.8 & 83.6 & 64.6 & 79.0 \\
      \quad w/ TA & 0.06 & \multirow{-4}{*}{\rotatebox[origin=c]{90}{$\sim$645$\times$}} & 84.2 & 85.3 & 83.2 & 81.2 & 72.5/74.0 & 83.4 & 66.1 & 78.7 \\
    \midrule
      MiniLM\textsubscript{\sf 3L;96H} & 0.04 & & 83.1 & 84.1 & 81.8 & 79.7 & 69.7/70.8 & 79.2 & 63.2 & 76.5 \\
      \quad w/ TA & 0.04 & & 83.0 & 83.2 & 81.2 & 80.3 & 69.3/70.7 & 81.8 & 60.7 & 76.3 \\
      \textsc{MiniMoE}\textsubscript{\sf 3L;96H} & 0.04 & & 83.0 & 84.5 & 82.7 & 81.1 & 71.7/72.8 & 82.1 & 63.9 & 77.7 \\
      \quad w/ TA & 0.04 & \multirow{-4}{*}{\rotatebox[origin=c]{90}{$\sim$968$\times$}} & 83.8 & 84.4 & 83.0 & 81.2 & 71.8/72.8 & 82.4 & 63.9 & 77.9 \\
    \bottomrule
    \end{tabular}
    \end{adjustbox}
    \label{tab_d}
\end{table*}

\section{Related Work}
\label{app_e}

\paragraph{Knowledge Distillation} 

Distillation~\citep{HintonVD15} is a de facto way to compression~\citep{BucilaCN06} LMs by transferring the knowledge of LMs to small language models. During the distillation, a small language model serves as a student and treats a LM as a teacher to learn from. There are three lines of work in LM distillation: firstly, task-specific distillation~\citep{SunCGL19,LiLZXYJ20,SunGFCWL20,ParkKY21,HouHSJCL20,XiaZC22} that conducts distillation on a specific task at finetuning stage; secondly, task-agnostic distillation~\citep{Turc19,Sanh19,SunYSLYZ20,WangBHDW21} that conducts distillation at pretraining stage; and thirdly, two-stage distillation~\citep{JiaoYSJCL0L20} that combines the power of both task-agnostic and -specific distillation. Though these methods realize promising performance when distilling LMs like BERT\textsubscript{\sf base}, they can come short of scalability to LMs like BERT\textsubscript{\sf large} especially when the student is of a small scale. In fact, driven by recent observations~\citep{WangW0B0020,Zhang22,MirzadehFLLMG20,ChoH19}, distillation with a small student can be faced with two deficiencies due to the large capacity gap. A few studies including teacher assistant-based~\citep{MirzadehFLLMG20,Zhang22} and student-friendly~\citep{ParkCJKH21,ZhouXM22} distillation can alleviate the first but fail to resolve the second. It is noteworthy that some work states they can tackle both deficiencies for vision models~\citep{ZhuW21a,Zhao22}, but preliminary studies have found that they are either expensive or not capable of LMs. In our work, we follow the line of task-agnostic distillation of LMs and aims at lifting both efficiencies for the first time.

\paragraph{Mixture of Experts} 

Based on the idea of conditional computation~\citep{BengioBPP15}, MoE layer is proposed to scale-up LMs in a sparsely activated fashion~\citep{ShazeerMMDLHD17}. There are diverse designs to achieve the sparse routing, such as gating~\citep{ShazeerCPTVKHLH18} and hashing~\citep{RollerSSW21}, with necessary balance constraints~\citep{LepikhinLXCFHKS21}. MoE layers are then joined to LMs in the past one or two years~\citep{Fedus21,DuHDTLXKZYFZFBZ22}. Owing to the sparse activation property, the scales of LMs are significantly increased with only minor losses in compute efficiency on modern GPU devices so that the underneath scaling laws can be uncovered in a comparably cheap manner~\citep{He21,RajbhandariLYZA22}. In our work, we are impelled by the merits of MoE, and propose a \textsc{MiniMoE} so that the capacity of the student can be enlarged without much inference overhead increment. \textsc{MiniMoE} can be similar to a certain stream of methods~\citep{ZhangL00S022,ZuoZLHZC22} that pursue accelerating LMs via precisely moefying them. Nonetheless, the moefication process is exerted to LMs with limited inference compute improvements compared to those advanced by \textsc{MiniMoE}. Note that there are emergent work exploring compressing MoE LMs~\citep{Xue22} to dense students, which is walking down the same street in the opposite side since we instead focus on compressing dense LMs to MoE students.

\section{Results on BERT\textsubscript{\sf xlarge}}
\label{app_f}

LM distillation, under either the task-agnostic setting as in our paper or the task-specific setting, has seldom been investigated to distil LMs larger than BERT\textsubscript{\sf large}. Even worse, there is only little work has been investigated to distil BERT\textsubscript{\sf large} under the task-agnostic setting.

In the main results, we just follow the paces of the task-agnostic setting, not only due to the huge scales of larger LMs like T5 and GPT3 but also due to that task-agnostic LM distillation requires the access to the original pretraining data of usually vast volume. What's more, larger LMs like T5 can be incomparable to BERT owing to the architectural difference, and existing task-agnostic methods including ours may easily fail.

Regarding all the considerations mentioned above, however, we try to check the existence of the curse of capacity gap and examine \textsc{MiniMoE} under a comparably larger-scale setting, i.e., Chinese BERT\textsubscript{\sf base} v.s. BERT\textsubscript{\sf xlarge} on some datasets from CLUE~\citep{XuHZLCLXSYYTDLS20} (which can be viewed as the Chinese GLUE). These datasets include a topic classification dataset TNews, a similar question matching dataset AFQMC, and a natural language inference dataset OCNLI. The preliminary results are shown in Table~\ref{tab_e}. As far as we know, while English BERT\textsubscript{\sf xlarge} with more than one billion parameters trained by Nvidia Megatron~\citep{Shoeybi19} is not publicly available, Chinese BERT\textsubscript{\sf xlarge} can be easily downloaded through huggingface.\footnote{\url{https://huggingface.co/IDEA-CCNL/Erlangshen-MegatronBert-1.3B.}} It is noteworthy that Chinese BERT\textsubscript{\sf base} is trained on Chinese Wikipedia ($\sim$15G) while Chinese BERT\textsubscript{\sf xlarge} is trained on Wudao Corpus ($\sim$300G)~\citep{YuanZDDLCZYT21}. We use Wikipedia data as the default choice for distillation, but Wudao data seems to be a more suitable (though not that fair) one for distilling Chinese BERT\textsubscript{\sf xlarge} as we have found that Wikipedia could not make the distillation converge properly. Painfully, it consumes around one week to achieve one epoch of distilling Chinese BERT\textsubscript{\sf xlarge} using Wudao in contrast to five epochs of distilling Chinese BERT\textsubscript{\sf base} on Wikipedia in one day. The results show that Chinese BERT\textsubscript{\sf xlarge} is cursed to realize better students than Chinese BERT\textsubscript{\sf base} does, and \textsc{MiniMoE} has the potential to lift the curse under the larger-scale setting.

\begin{table*}[ht]
    \caption{The results of comparison between distilling Chinese BERT\textsubscript{\sf base} and BERT\textsubscript{\sf xlarge}.}
    \begin{adjustbox}{width=0.65\textwidth,center}
    \begin{threeparttable}
    \begin{tabular}{ll|ccc|c}
    \toprule
      \textbf{Method} & \textbf{Teacher} & \makecell[c]{\textbf{TNews}\\\textbf{Acc}} & \makecell[c]{\textbf{AFQMC}\\\textbf{Acc}} & \makecell[c]{\textbf{OCNLI}\\\textbf{Acc}} & \makecell[c]{\textbf{CLUE}\\\textbf{Score}} \\
    \midrule
      \multirow{2}{*}{Teacher} & BERT\textsubscript{\sf base} & 57.0 & 74.8 & 75.4 & 69.1 \\
      & BERT\textsubscript{\sf xlarge}$\textcolor[rgb]{0,0.7,0}{\boldsymbol{\Uparrow}}$ & 60.0 & 76.1 & 79.2 & 71.7 \\
      \cmidrule{2-6}
      \multirow{2}{*}{MiniLM\textsubscript{\sf 6L;384H}} & BERT\textsubscript{\sf base} & 55.5 & 72.0 & 71.0 & 66.2 \\
      & BERT\textsubscript{\sf xlarge}$\textcolor[rgb]{1,0,0}{\boldsymbol{\Downarrow}}$ & 54.9 & 70.7 & 69.9 & 65.2 \\
      \cmidrule{2-6}
      \rowcolor{green!20} & BERT\textsubscript{\sf base} & 55.9 & 72.9 & 70.8 & 66.5 \\
      \rowcolor{green!20} \multirow{-2}{*}{ \textsc{MiniMoE}\textsubscript{\sf 6L;384H}} & BERT\textsubscript{\sf xlarge}$\textcolor[rgb]{0,0.7,0}{\boldsymbol{\Uparrow}}$\tnote{1} & 56.7 & 72.4 & 71.0 & 66.7 \\
    \bottomrule
    \end{tabular}
    \begin{tablenotes}
      \item [1] $\textcolor[rgb]{0,0.7,0}{\boldsymbol{\Uparrow}}$ is used to indicate the deficiency is tackled on CLUE, otherwise $\textcolor[rgb]{1,0,0}{\boldsymbol{\Downarrow}}$ is used.
     \end{tablenotes}
    \end{threeparttable}
    \end{adjustbox}
    \label{tab_e}
\end{table*}

\section{Potential of Memory-efficient \textsc{MiniMoE}}
\label{app_g}

One may argue that \textsc{MiniMoE} introduces much more memory consumption than MiniLM does, largely limiting the application scenarios for memory-sensitive devices (e.g., mobile devices).

However, there is no free lunch to enlarge the capacity of the student. We should claim that, in order to increase the capacity, memory/space consumption is a cheaper choice (e.g., more experts) than latency/time consumption (e.g., more operations), and this is potentially the reason why large LMs like PaLM~\citep{Chowd22} and FLAN~\citep{Chung22} could become so popular. We should also highlight that scenarios that require rather limited memory consumption (e.g., mobile scenarios) is currently not (though can be in the near future) the main concern of LMs. In contrast, LMs are usually served in GPU scenarios, where memory/space is easy to access.

Luckily, we find a potential path to address the memory efficiency concern based on the idea of parameter decomposition (e.g., SVD). While embedding parameter decomposition is a general way to reduce the number of parameters for embeddings and could not make \textsc{MiniMoE} as memory-efficient as MiniLM. We uncover that, without much performance sacrifice, transformer parameter decomposition in \textsc{MiniMoE} can be easier in comparison with that in MiniLM owing to the sparse activation property of MoE. That is, transformer parameters in \textsc{MiniMoE} have lower ranks than those in MiniLM, and this can be shown by analyzing the magnitudes of the normalized singular values using SVD. The preliminary results of the output matrices of the last FFN layers separately from MiniLM\textsubscript{\sf 3L;384H} and \textsc{MiniMoE}\textsubscript{\sf 3L;384H} are shown in Table~\ref{tab_f}.

\begin{table*}[ht]
    \centering
    \caption{The SVD analysis to show the potential of memory-efficient \textsc{MiniMoE}.}
    \adjustbox{width=\textwidth,center}{
    \begin{tabular}{lllll}
    \toprule
      \textbf{Method} & \textbf{\%Value$>$0.2} & \textbf{\%Value$>$0.1} & \textbf{\%Value$>$0.05} & \textbf{Trm Params (Value$>$0.1)} \\
    \midrule
      MiniLM\textsubscript{\sf 3L;384H} dense & 315/384=82\% & 356/384=93\% & 373/384=97\% & 5.3M$\to$5.1M \\
    \midrule
      MiniMoE\textsubscript{\sf 3L;384H} expert \#1 & 6/384=2\% & 82/384=21\% & 275/384=72\% & - \\
      MiniMoE\textsubscript{\sf 3L;384H} expert \#2 & 34/384=9\% & 220/384=57\% & 361/384=94\% & - \\
      MiniMoE\textsubscript{\sf 3L;384H} expert \#3 & 15/384=4\% & 175/384=46\% & 338/384=88\% & - \\
      MiniMoE\textsubscript{\sf 3L;384H} expert \#4 & 24/384=6\% & 200/384=52\% & 357/384=93\% & - \\
    \midrule
      MiniMoE\textsubscript{\sf 3L;384H} all experts & 79/384/4=5\% & 677/384/4=44\% & 1331/384/4=87\% & 16.4M$\to$8.2M \\
    \bottomrule
    \end{tabular}
    }
    \label{tab_f}
\end{table*}

With this finding, \textsc{MiniMoE} can compress more parameters than MiniLM does using parameter decomposition and finally yield a similar memory efficiency to that of MiniLM. 

We also explore another complementary solution that views \textsc{MiniMoE} as teacher assistant and further distils from \textsc{MiniMoE} to its dense counterpart. The results are shown in Table~\ref{tab_g}, implying that this only results in an acceptable performance degradation on large datasets like MNLI and SST-2 but undesired performance degradation on small datasets like RTE.

\begin{table*}[ht]
    \centering
    \caption{The results of further distilling from \textsc{MiniMoE} to its dense counterpart.}
    \begin{adjustbox}{width=0.6\textwidth,center}
    \begin{tabular}{lccc}
    \toprule
      \textbf{Method} & \textbf{MNLI-m/mm} & \textbf{SST-2} & \textbf{RTE} \\
    \midrule
      MiniLM\textsubscript{\sf 3L;384H} & 77.8/78.4 & 89.1 & 66.1 \\
      \textsc{MiniMoE}\textsubscript{\sf 3L;384H} & 78.2/78.7 & 89.3 & 67.0 \\
      \textsc{MiniMoE}\textsubscript{\sf 3L;384H}$\Rightarrow$MiniLM\textsubscript{\sf 3L;384H}  & 78.1/78.4 & 89.5 & 64.3 \\
    \bottomrule
    \end{tabular}
    \end{adjustbox}
    \label{tab_g}
\end{table*}

\begin{table*}[ht]
    \centering
    \caption{The results of applying vision distillation methods upon BERT\textsubscript{\sf base}.}
    \begin{adjustbox}{width=0.35\textwidth,center}
    \begin{tabular}{lclc}
    \toprule
      \textbf{Method} & \textbf{GLUE} & \textbf{Method} & \textbf{GLUE} \\
    \midrule
      KD\textsubscript{\sf 2L} & 72.9 & KD\textsubscript{\sf 4L} & 81.8 \\
      \quad w/ TA & 73.4 & \quad w/ TA & 82.1 \\
      DeKD\textsubscript{\sf 2L} & 72.7 & DeKD\textsubscript{\sf 4L} & 81.6 \\
    \bottomrule
    \end{tabular}
    \end{adjustbox}
    \label{tab_6}
\end{table*}

\section{Failure of Vision Method}
\label{app_h}

We examine in a preliminary study the effectiveness of one of the vision model distillation methods~\citep[DeKD,][]{Zhao22} which can lift the curse of capacity gap. From the results in Table~\ref{tab_6}, we unfortunately discover that DeKD can only give comparable performance in distilling BERT\textsubscript{\sf base}, which even lags behind KD w/ TA. It hints that vision model distillation methods are not that capable of LMs.

\end{document}